\documentclass{article}

    \usepackage[preprint]{neurips_2024}

\usepackage[utf8]{inputenc} % allow utf-8 input
\usepackage[T1]{fontenc}    % use 8-bit T1 fonts
\usepackage{hyperref}       % hyperlinks
\usepackage{url}            % simple URL typesetting
\usepackage{booktabs}       % professional-quality tables
\usepackage{amsfonts}       % blackboard math symbols
\usepackage{nicefrac}       % compact symbols for 1/2, etc.
\usepackage{microtype}      % microtypography
\usepackage{xcolor}         % colors

\usepackage{amsmath,amsthm,amssymb}

\usepackage{graphicx}
\usepackage{subcaption}
\usepackage{multicol,lipsum}
\usepackage{multirow}
\usepackage{wrapfig}
\usepackage{array}
\usepackage{amsmath}
\usepackage{amssymb}
\usepackage{mathtools}
\usepackage{amsthm}
\usepackage{algorithm, algorithmic}
\usepackage{enumitem}
\usepackage{natbib}
\theoremstyle{plain}
\newtheorem{theorem}{Theorem}[section]

\newtheorem{lemma}[theorem]{Lemma}

\newtheorem{definition}[theorem]{Definition}
\newtheorem{assumption}[theorem]{Assumption}
\theoremstyle{remark}

\usepackage{amsmath,amsfonts,bm, amsthm, mathrsfs, mathtools}

\newcommand*\diff{\mathop{}\!\mathrm{d}}

\def\1{\bm{1}}

\DeclareMathAlphabet{\mathsfit}{\encodingdefault}{\sfdefault}{m}{sl}
\SetMathAlphabet{\mathsfit}{bold}{\encodingdefault}{\sfdefault}{bx}{n}

\def\gD{{\mathcal{D}}}

\def\gO{{\mathcal{O}}}

\def\sR{{\mathbb{R}}}

\newcommand{\E}{\mathbb{E}}

\usepackage{autonum}

\title{RAMP: Boosting Adversarial Robustness Against Multiple $l_p$ Perturbations for Universal Robustness}

\author{Enyi Jiang, Gagandeep Singh  \\
	University of Illinois Urbana-Champaign\\
	\texttt{\{enyij2,ggnds\}@illinois.edu} \\
}

\begin{document}

\maketitle

\begin{abstract}
  Most existing works focus on improving robustness against adversarial attacks bounded by a single $l_p$ norm using adversarial training (AT). However, these AT models' multiple-norm robustness (union accuracy) is still low, which is crucial since in the real-world an adversary is not necessarily bounded by a single norm.
%training with multiple-norm contributes to better \emph{universal robustness} towards corruptions and other adversaries.
The tradeoffs among robustness against multiple $l_p$ perturbations and accuracy/robustness make obtaining good union and clean accuracy challenging. We design a logit pairing loss to improve the union accuracy by analyzing the tradeoffs from the lens of distribution shifts. We connect natural training (NT) with AT via gradient projection, to incorporate useful information from NT into AT, where we empirically and theoretically show it moderates the accuracy/robustness tradeoff. We propose a novel training framework \textbf{RAMP}, to boost the robustness against multiple $l_p$ perturbations. \textbf{RAMP} can be easily adapted for robust fine-tuning and full AT. For robust fine-tuning, \textbf{RAMP} obtains a union accuracy up to $53.3\%$ on CIFAR-10, and $29.1\%$ on ImageNet. For training from scratch, \textbf{RAMP} achieves a union accuracy of $44.6\%$ and good clean accuracy of $81.2\%$ on ResNet-18 against AutoAttack on CIFAR-10. Beyond multi-norm robustness \textbf{RAMP}-trained models achieve superior \textit{universal robustness}, effectively generalizing against a range of unseen adversaries and natural corruptions.

% identify the key tradeoff pair among $l_p$ attacks to boost efficiency and   
  % \keywords{Adversarial Robustness \and Pre-training and Fine-tuning \and Distribution Shifts}
\end{abstract}

\section{Introduction}
% provide the background and define the problem
Though deep neural networks (DNNs) demonstrate superior performance in various vision applications, they are vulnerable against adversarial examples~\citep{goodfellow2014explaining, kurakin2018adversarial}. Adversarial training (AT)~\citep{tramer2017ensemble,madry2017towards} which works by injecting adversarial examples into training for enhanced robustness, is currently the most popular defense. However, most AT methods address only a \emph{single} type of perturbation~\citep{Wang2020mart, wu2020awp, carmon2019unlabeled, gowal2020uncovering, raghunathan2020understanding, zhang2021_GAIRAT, debenedetti2022adversarially, peng2023robust,pmlr-v202-wang23ad}. An $l_\infty$ robust model may not be robust against $l_p (p \neq \infty)$ attacks. Also, enhancing robustness against one perturbation type can sometimes increase vulnerability to others~\citep{engstrom2017rotation, schott2018towards}. On the contrary, training a model to be robust against multiple $l_p$ perturbations is crucial as it reflects real-world scenarios~\citep{sharif2016accessorize, eykholt2018robust,song2018physical,athalye2018synthesizing} where adversaries can use multiple $l_p$ perturbations. We show that multi-norm robustness is the key to improving generalization against other threat models~\citep{croce2022eat}. For instance, we show it enables robustness against perturbations not easily defined mathematically, such as image corruptions and unseen adversaries~\citep{wong2020learning}. %We show being robust against multiple-norm perturbations is the key to improving generality to unseen adversaries and corruptions (Section~\ref{generality-main}). 

% Our new results on generalizing to other perturbations and corruptions (see tables in response to reviewers 6TV5, 3Hj9, and Qhq5) confirm that point with the best robustness-accuracy tradeoff among $l_p$ pretrained, SOTA image corruption, E-AT, and MAX models. 

% Training models robust against multiple $l_p$ perturbations simultaneously is harder than single perturbations~\citep{croce2020robustbench}.

% The most popular defense at present is adversarial training (AT)~\cite{tramer2017ensemble,madry2017towards}; adversarial examples are generated and injected into the training process for better robustness. Most AT methods only consider a \emph{certain} type of perturbation \emph{at a time}~\citep{Wang2020mart, wu2020awp, carmon2019unlabeled, gowal2020uncovering, raghunathan2020understanding, zhang2021_GAIRAT, debenedetti2022adversarially, peng2023robust,pmlr-v202-wang23ad}. For example, an $l_\infty$ robust model has low robustness against $l_p (p \neq \infty)$ attacks (Figure~\ref{fig:l1-linf-pretrain}). It remains unclear how one can train a model to become more robust against multiple $l_p$ perturbations, instead of one~\citep{croce2020robustbench}. 

% challenges of ATMP, what makes it harder than single perturbation

% too many distribution shifts
Two main challenges exist for training models robust against multiple perturbations: (i) tradeoff among robustness against different perturbation models~\citep{tramer2019atmp} and (ii) tradeoff between accuracy and robustness~\citep{zhang2019trades, raghunathan2020understanding}. Adversarial examples induce a shift from the original distribution, causing a drop in clean accuracy with AT~\citep{xie2020adversarial, benz2021batch}. The distinct distributions created by $l_1, l_2, l_\infty$ adversarial examples make the problem even more challenging. Through a finer analysis of the distribution shifts caused by these adversaries, we propose the \textbf{RAMP} framework to efficiently boost the \textbf{R}obustness \textbf{A}gainst \textbf{M}ultiple \textbf{P}erturbations. \textbf{RAMP} can be used for both fine-tuning and training from scratch. It utilizes a novel logit pairing loss on a certain pair and connects NT with AT via gradient projection~\citep{jiang2023fedgp} to improve union accuracy while maintaining good clean accuracy and training efficiency.

\noindent\textbf{Logit pairing loss.} We visualize the changing of $l_1, l_2, l_\infty$ robustness when fine-tuning a $l_\infty$-AT pre-trained model in Figure~\ref{fig:l1linf-tradeoff-intro} using the CIFAR-10 training dataset. The DNN loses substantial robustness against $l_\infty$ attack after only $1$ epoch of fine-tuning: $l_1$ fine-tuning and E-AT~\citep{croce2022eat} (red and yellow histograms under Linf category) both lose significant $l_{\infty}$ robustness (compared with blue histogram under Linf category). Inspired by this observation, we devise a new logit pairing loss for a $l_q - l_r$ tradeoff pair to attain better union accuracy, which enforces the logit distributions of $l_q$ and $l_r$ adversarial examples to be close, specifically on the correctly classified $l_q$ subsets. In comparison, our method (green histogram under Linf and union categories) preserves more $l_\infty$ and union robustness than others after $1$ epoch. We show this technique works on larger models and datasets (Section~\ref{sec:robust-finetune}). 

\begin{wrapfigure}{r}{0.52\textwidth}
% \vspace*{-\baselineskip}
\vspace{-0.8cm}
         % \begin{center}
         \centering
         \includegraphics[width=\linewidth]{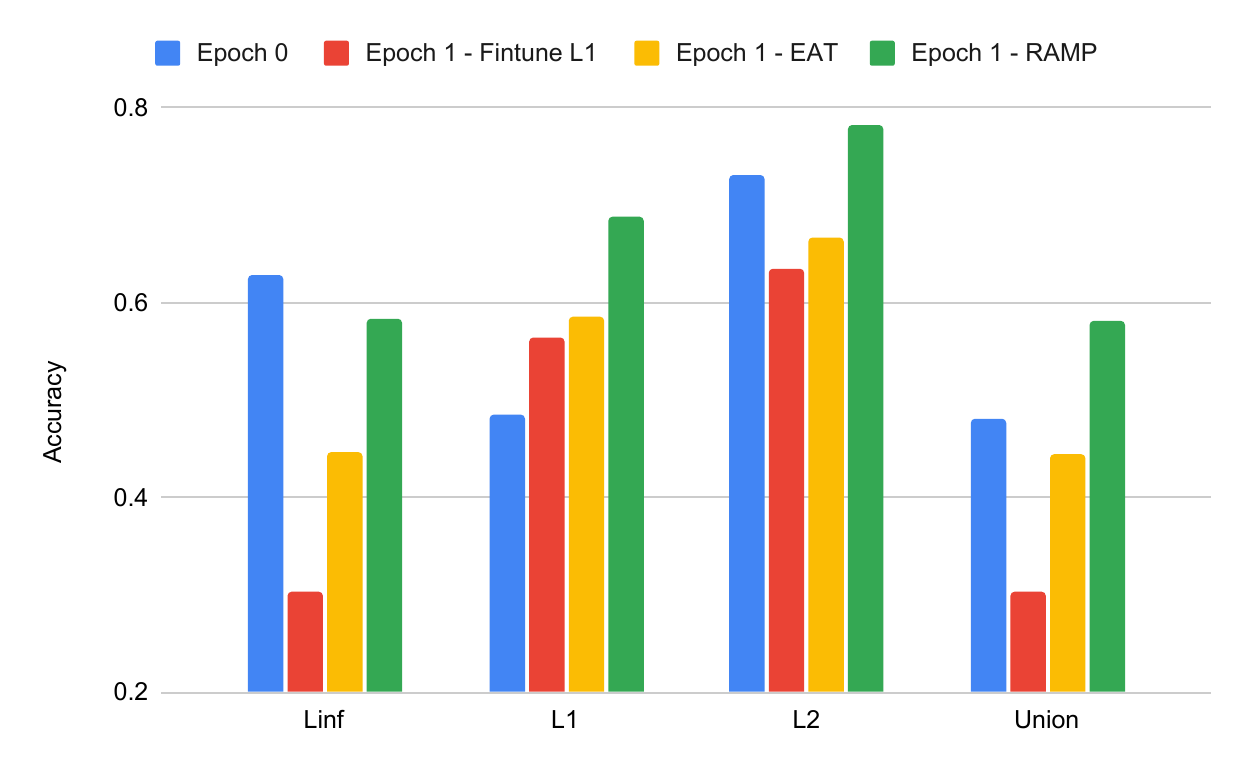}
         % \end{center}
         \caption{\label{fig:l1linf-tradeoff-intro} \textbf{Multiple-norm tradeoff with robust fine-tuning}: We observe that fine-tuning on $l_\infty$-AT model using $l_1$ examples drastically reduces $l_\infty$ robustness. \textbf{RAMP} preserves more $l_\infty$ and union robustness.} 
         % For a single $l_p$ attack, correct predictions are colored with cyan and incorrect with magenta. For the union accuracy, the yellow color refers to the robustness against three perturbations. The rows show the multiple-norm robustness of (1) $l_\infty$-AT model (epoch 0), (2) after $1$ epoch of fine-tuning on $l_1$ examples, (3) after $1$ epoch of E-AT fine-tuning, (4) after $1$ epoch of \textbf{RAMP} fine-tuning.
\vspace*{-\baselineskip}
\end{wrapfigure}

% method part 2: accuracy/robustness tradeoff
\noindent\textbf{Connect natural training (NT) with AT.}
We explore the connections between NT and AT to obtain a better accuracy/robustness tradeoff. We find that NT can help with adversarial robustness: useful information in natural distribution can be extracted and leveraged to achieve better robustness. To this end, we compare the similarities of model updates of NT and AT \emph{layer-wise} for each epoch, where we find and incorporate useful NT components into AT via gradient projection (GP), as outlined in Algorithm~\ref{algo:atgp}. In Figure~\ref{algo:atgp} and Section~\ref{sec:robust-pretrain}, we empirically and theoretically show this technique strikes a better balance between accuracy and robustness, for both single and multiple $l_p$ perturbations. We provide a theoretical analysis of why GP works for adversarial robustness in Theorem~\ref{thm:convergence} \& ~\ref{thm:error-GP}.

% To connect NT with AT more effectively, 

% summary of contributions
\noindent\textbf{Main contributions}: 
\begin{itemize}[leftmargin=*]
    % \item We analyze tradeoffs between robustness against different $l_p$ perturbations and accuracy/robustness from the lens of distribution shifts, with good efficiency at the same time.
    \item We design a new logit pairing loss to mitigate the $l_q - l_r$ tradeoff for better union accuracy, by enforcing the logit distributions of $l_q$ and $l_r$ adversarial examples to be close.
    \item We empirically and theoretically show that connecting NT with AT via gradient projection better balances the accuracy/robustness tradeoff for $l_p$ perturbations, compared with standard AT.
    \item \textbf{RAMP} achieves good union accuracy, accuracy-robustness tradeoff, and generalizes better to diverse perturbations and corruptions (Section~\ref{generality-main}) achieving superior \emph{universal robustness} ($75.5\%$ for common corruption and $26.1\%$ union accuracy against unseen adversaries). \textbf{RAMP} fine-tuned DNNs achieve union accuracy up to $53.3\%$ on CIFAR-10, and $29.1\%$ on ImageNet. \textbf{RAMP} achieves a $44.6\%$ union accuracy and good clean accuracy on ResNet-18 against AutoAttack on CIFAR-10. 
    % We will publically release our code upon acceptance of this work.
\end{itemize}
% identify the key $l_q - l_r$ tradeoff pair to improve the efficiency, 

% We will publically release our code upon acceptance of this work. 
Our code is available at ~\url{https://github.com/uiuc-focal-lab/RAMP}.

% \centering
     % \begin{subfigure}[t]{0.45\textwidth}
     %     \centering
     %     \includegraphics[width=0.65\textwidth]{images/l1_linf_tradeoff.pdf}
     %     \caption{\textbf{Multiple-norm tradeoff with $l_1, l_2, l_\infty$ AT pre-trained models} (colored in blue, red, and orange respectively), with respect to their clean accuracy (y-axis), and robust accuracy (y-axis) against $l_\infty, l_2, l_1$ attacks on the x-axis. Here we use $\epsilon_1=12, \epsilon_2=0.5, \epsilon_\infty=\frac{8}{255}$.  }
     %     \label{fig:l1-linf-pretrain}
     % \end{subfigure}
     % \hfill
     % \begin{subfigure}[t]{0.52\textwidth}
\section{Related Work}
\noindent \textbf{Adversarial training (AT).} Adversarial Training (AT) usually employs gradient descent to discover adversarial examples, incorporating them into training for enhanced adversarial robustness~\citep{tramer2017ensemble,madry2017towards}. Numerous works focus on improving robustness by exploring the trade-off between robustness and accuracy~\citep{zhang2019trades,Wang2020mart}, instance reweighting~\citep{zhang2021_GAIRAT}, loss landscapes~\citep{wu2020awp}, wider/larger architectures~\citep{gowal2020uncovering, debenedetti2022adversarially}, data augmentation~\citep{carmon2019unlabeled, raghunathan2020understanding}, and using synthetic data~\citep{peng2023robust,pmlr-v202-wang23ad}. However, these methods often yield DNNs robust against a \emph{single} perturbation type while remaining vulnerable to other types.

\noindent \textbf{Robustness against multiple perturbations.} \citet{tramer2019atmp,kang2019transfer} observe that robustness against $l_p$ attacks does not necessarily transfer to other $l_q$ attacks ($q\neq p$). Previous studies~\citep{tramer2019atmp,maini2020msd,madaan2021sat,croce2022eat} modified Adversarial Training (AT) to enhance robustness against multiple $l_p$ attacks, employing average-case~\citep{tramer2019atmp}, worst-case~\citep{tramer2019atmp, maini2020msd}, and random-sampled~\citep{madaan2021sat, croce2022eat} defenses. There are also works~\citep{nandy2020approximate, liu2020towards, xu2021mixture, xiao2022adaptive, maini2022perturbation} using preprocessing, ensemble methods, mixture of experts, and stability analysis to solve this problem. Ensemble models and preprocessing methods are weakened since their performance heavily relies on correctly classifying or detecting various types of adversarial examples. Also, prior works are hard to scale to larger models and datasets, e.g. ImageNet, due to the efficiency issue. Furthermore,~\citet{croce2022eat} devise Extreme norm Adversarial Training (E-AT) and fine-tune a $l_p$ robust model on another $l_q$ perturbation to quickly make a DNN robust against multiple $l_p$ attacks. However, E-AT does not adapt to varying epsilon values. Our work demonstrates that the suboptimal tradeoff observed in prior studies can be improved with our proposed framework.
%achieved by prior works can be suboptimal. %Thus, we conduct a finer analysis of this tradeoff, which leads us to the new logit pairing loss for better union accuracy.

\noindent \textbf{Logit pairing in adversarial training.} Adversarial logit pairing methods encourage logits for pairs of examples to be similar~\citep{kannan2018logitpair,engstrom2018evaluatinglogit}. People apply this technique to both clean images and their adversarial counterparts, to devise a stronger form of adversarial training. In our work, we devise a novel logit pairing loss to train a DNN originally robust against $l_p$ attack to become robust against another $l_q(q\neq p)$ attack on the correctly predicted $l_p$ subsets, which helps gain better union accuracy.

% not sure how to state the related work on GP
\noindent \textbf{Adversarial versus distributional robustness.} ~\citet{sinha2018certifiable} theoretically studies the AT problem through distributional robust optimization.~\citet{mehrabi2021fundamental} establishes a pareto-optimal tradeoff between standard and adversarial risks by perturbing the test distribution. Other works explore the connection between natural and adversarial distribution shifts~\citep{moayeri2022explicit, alhamoud2023generalizability}, assessing transferability and generalizability of adversarial robustness across datasets. However, little research delves into distribution shifts induced by $l_1, l_2, l_\infty$ adversarial examples and their interplay with the robustness-accuracy tradeoff~\citep{zhang2019trades, yang2020closer, rade2021reducing}. Our work, inspired by recent domain adaptation techniques~\citep{jiang2023fda, jiang2023fedgp}, designs a logit pairing loss and utilizes model updates from NT via GP to enhance adversarial robustness. We show that GP adapts to both single and multi-norm scenarios.

% identifies a key tradeoff pair among multiple perturbations and 

% They measure and explore the transferability and generalizability of adversarial robustness among different datasets with distribution shifts.
% with better accuracy/robustness tradeoff. 
% - training on adversarial examples degrades the original clean accuracy
\section{AT against Multiple Perturbations}
We consider a standard classification task with samples $\{(x_i, y_i)\}_{i=0}^{N}$ from an empirical data distribution $\widehat \gD_n$; we have input images $x \in \mathbb{R}^{d}$ and corresponding labels $y \in \mathbb{R}^k$. Standard training aims to obtain a classifier $f$ parameterized by $\theta$ to minimize a loss function $\mathcal{L}: \mathbb{R}^k \times \mathbb{R}^k \rightarrow \mathbb{R}$ on $\widehat \gD_n$. Adversarial training (AT)~\citep{madry2017towards,tramer2017ensemble} aims to find a DNN robust against adversarial examples. It is framed as a min-max problem where a DNN is optimized using the worst-case examples within an adversarial region around each $x_i$. Different types of adversarial regions $B_p (x, \epsilon_p)= \{x^{\prime} \in \mathbb{R}^d: \|x^{\prime} - x \|_p \leq \epsilon_p \}$ can be defined around a given image $x$ using various $l_{p}$-based perturbations. %Therefore, the goal is to find the DNN parameter $\theta$ minimizing the risk on the worst-case examples $x^{\prime}$ within the ball region .
Formally, we can write the optimization problem of AT against a certain $l_p$ attack as follows:

\begin{equation}
    \min_{\theta} \mathbb{E}_{(x,y)\sim \widehat \gD_n}\left[\max_{x^{\prime} \in B_p (x, \epsilon_p)} \mathcal{L}(f(x^{\prime}),y)\right]
\end{equation}

The above optimization is only for certain $p$ values and is usually vulnerable to other perturbation types. To this end, prior works have proposed several approaches to train the network robust against multiple perturbations ($l_1, l_2, l_{\infty}$) at the same time. We focus on the union threat model $\Delta = B_1 (x, \epsilon_1) \cup B_2 (x, \epsilon_2) \cup B_\infty (x, \epsilon_\infty)$ which requires the DNN to be robust within the $l_1, l_2, l_{\infty}$ adversarial regions simultaneously~\citep{croce2022eat}. Union accuracy is then defined as the robustness against $\Delta_{(i)}$ for each $x_i$ sampled from $\mathcal{D}$. In this paper, similar to the prior works, we use union accuracy as the main metric to evaluate the multiple-norm robustness. Apart from that, we define \emph{universal robustness} as the generalization ability against a range of unseen adversaries and common corruptions. Specifically, we have average accuracy across five severity levels for common corruption and union accuracy against a range of unseen adversaries used in~\citet{laidlaw2020perceptual}.
% There are three main categories of defenses: 1. Worst-case, 2. Average-case, and 3. Random-sampled defenses.

\noindent \textbf{Worst-case defense} follows the following min-max optimization problem to train DNNs using the worst-case example from the $l_1, l_2, l_{\infty}$ adversarial regions:

\begin{equation}
   \min_{\theta} \mathbb{E}_{(x,y)\sim \widehat \gD_n}\left[\max_{p \in \{1,2,\infty\}} \max_{x^{\prime} \in B_p (x, \epsilon_p)} \mathcal{L}(f(x^{\prime}),y)\right]
\end{equation}

MAX~\citep{tramer2019atmp} and MSD~\citep{maini2020msd} fall into this category. Finding worst-case examples yields a good union accuracy but results in a loss of clean accuracy as the distribution of generated examples is different from the clean data distribution.

% where MSD is more fine-grained as it finds the worst-case examples with the highest loss during each step of inner maximization

\noindent\textbf{Average-case defense} train DNNs using the average of the $l_1, l_2, l_{\infty}$ worst-case examples:

\begin{equation}
     \min_{\theta} \mathbb{E}_{(x,y)\sim \widehat \gD_n}\left[\mathbb{E}_{p \in \{1,2,\infty\}} \max_{x^{\prime} \in B_p (x, \epsilon_p)} \mathcal{L}(f(x^{\prime}),y)\right]
\end{equation}

AVG~\citep{tramer2019atmp} is of this type. This method generally leads to good clean accuracy but suboptimal union accuracy as it does not penalize worst-case behavior within the $l_1, l_2, l_{\infty}$ regions.

\noindent \textbf{Random-sampled defense.} The defenses mentioned above lead to a high training cost as they compute multiple attacks for each sample. SAT~\citep{madaan2021sat} and E-AT~\citep{croce2022eat} randomly sample one attack out of each type at a time, contributing to a similar computational cost as standard AT on a single perturbation model. They achieve a slightly better union accuracy compared with AVG and relatively good clean accuracy. However, they are not better than worst-case defenses for multiple-norm robustness, since they do not consider the strongest attack within the union region all the time.

% 3. Union accuracy and efficiency: to achieve good union robustness, one needs to find the worst-case example among $l_p$ adversarial regions, to achieve a good \emph{union} accuracy.

% However, there is still a lack of analysis of the connection between the two distributions, such that we can leverage the usefulness of natural training into adversarial training (AT).

% Also, to train a DNN truly robust against all types, the DNN should be able to achieve a higher \emph{union} accuracy is much more important than \emph{average} accuracy. A high average accuracy may come from the individual robustness of different subsets of images in the dataset, rather than a DNN with robustness to all three types of attacks on the \emph{single} image.
\section{RAMP}
There are two main tradeoffs in achieving better union accuracy while maintaining good accuracy: 1. Among perturbations: there is a tradeoff among different attacks, e.g., a $l_\infty$ pre-trained AT DNN is not robust against $l_1, l_2$ perturbations, which makes the union accuracy harder to attain. Also, we observe there exists a main tradeoff pair of two attacks among the union over $l_1$, $l_2$, $l_\infty$ attacks. 2. Accuracy and robustness: all defenses lead to degraded clean accuracy. To address these tradeoffs, we study the problem from the lens of distribution shifts.
% 3. Performance and efficiency: finding the worst-case examples triples the computational cost compared with standard AT. 
%We observe that the different adversarial attacks lead to different distributons.  %Our \textbf{RAMP} aims to provide a comparatively satisfactory solution to these existing tradeoffs.

% the worst-case defense (highest union accuracy) finding the strongest adversary
% First, we visualize and analyze the tradeoff among perturbations and design a new loss to enforce union predictions. Then, we show how we can properly incorporate natural training with AT via gradient projection for better robustness. 

\noindent\textbf{Interpreting tradeoffs from the lens of distribution shifts.}
The adversarial examples with respect to an empirical data distribution $\widehat \gD_n$, adversarial region $B_p(x, \epsilon_p)$, and DNN $f_{\theta}$ generate a new adversarial distribution $\widehat \gD_a$ with samples $\{(x^{\prime}_i, y_i)\}_{i=0}^{N}$, that are correlated by adding certain perturbations but different from the original $\widehat \gD_n$. Because of the shifts between $\widehat \gD_n$ and $\widehat \gD_a$, DNN decreases performance on $\widehat \gD_n$ when we move away from it and towards $\widehat \gD_a$. Also, the distinct distributions created by multiple perturbations, $\widehat \gD_a^{l_1}$, $\widehat \gD_a^{l_2}$, $\widehat \gD_a^{l_\infty}$, contribute to the tradeoff among $l_1, l_2, l_\infty$ attacks. To address the tradeoff among perturbations while maintaining good efficiency, we focus on the distributional interconnections between $\widehat \gD_n$ and $\widehat \gD_a^{l_1}$, $\widehat \gD_a^{l_2}$, $\widehat \gD_a^{l_\infty}$. From the insights we get from above, we propose our framework \textbf{RAMP}, which includes (i) logit pairing to improve tradeoffs among multiple perturbations, and (ii) identifying and combining the useful DNN components using the model updates from NT and AT, to obtain a better robustness/accuracy tradeoff. 

% In Section~\ref{logit-pairing}, we study the distribution shifts among the adversarial regions from multiple perturbations. 

% In Section~\ref{nt-at}, we study the distribution shifts between $\widehat \gD_n$ and $\widehat \gD_a$ to mitigate the accuracy/robustness tradeoff.
% \textbf{RAMP} can achieve a better union accuracy for both robust fine-tuning and AT from random initialization (Section~\ref{sec:robust-finetune}) with relatively good efficiency.

\textbf{Identify the Key Tradeoff Pair.} We study the common case with $l_p$ norms $\epsilon_1=12, \epsilon_2=0.5, \epsilon_\infty=\frac{8}{255}$ on CIFAR-10~\citep{tramer2019atmp}. The distributions generated by the two strongest attacks show the largest shifts from $\widehat \gD_n$; also, they have the largest distribution shifts between each other because of larger and most distinct search areas. Thus, by calculating the ball volume~\citep{volball} for each attack, we select the two with the largest volumes as the key tradeoff pair. They refer to the strongest attack as the attacker has more search area. The attack with the smallest ball volume is mostly included by the convex hull of the other two stronger attacks~\citep{croce2022eat}. Here we identify $l_\infty - l_1$ as the key tradeoff pair.

\subsection{Logit Pairing for Multiple Perturbations}~\label{logit-pairing}
\noindent\textbf{Figure~\ref{fig:l1linf-tradeoff-intro}: Finetuning a $l_q$-AT model on $l_r$ examples reduces $l_q$ robustness.} To get a finer analysis of the $l_\infty - l_1$ tradeoff mentioned above, we visualize the changing of $l_1, l_2, l_{\infty}$ robustness of the training dataset when we fine-tune a $l_\infty$ pre-trained model with $l_1$ examples for $1$ epochs, as shown in Figure~\ref{fig:l1linf-tradeoff-intro}: x-axis represents the robustness against different attacks and y-axis is the accuracy. After 1 epoch of finetuning on $l_1$ examples or performing E-AT, we lose much $l_\infty$ robustness since blue/yellow histograms are much lower than the red histogram under the Linf category. RAMP preserves both $l_\infty$ and union robustness more effectively: the green histogram is higher than the red/yellow histogram under Linf and Union categories. Specifically, RAMP maintains 14\%, 28\% more union robustness than E-AT and $l_1$ fine-tuning. The above observations indicate the necessity of preserving more $l_q$ robustness as we adversarially fine-tune with $l_r$ adversarial examples on a $l_q$ pre-trained AT model, with $l_q - l_r$ as the key tradeoff pair, which inspires us to design our loss design with logit pairing. We want to enforce the \emph{union predictions} between $l_q$ and $l_r (q \neq r)$ attacks: bringing the predictions of $l_q$ and $l_r (q \neq r)$ close to each other, specifically on the correctly predicted $l_q$ subsets. Based on our observations, we design a new logit pairing loss to enforce a DNN robust against one $l_q$ attack to be robust against another $l_r (q \neq r)$ attack. 

%% Here we project each image into a 2D plane using the t-SNE plot and color them with correct or incorrect predictions. For an individual attack, correct classified points are colored in cyan, incorrect points are colored in magenta; for union accuracy, correct points are colored in yellow. Only after the $1$ epoch of robust fine-tuning, the new DNN loses the $l_{\infty}$ robustness fairly quickly with more points colored in magenta against $l_\infty$ attack.
% 

% For the example presented here, we have the main tradeoff between $l_1$ and $l_\infty$ with relatively large distribution shifts.

\noindent\textbf{Enforcing the Union Prediction via Logit Pairing.}
The $l_q - l_r (q\neq r)$ tradeoff leads us to the following principle to improve union accuracy: \emph{for a given set of images, when we have a DNN robust against some $l_q$ examples, we want it to be robust against $l_r$ examples as well.} This serves as the main insight for our loss design: we want to \emph{enforce} the logits predicted by $l_q$ and $l_r$ adversarial examples to be close, specifically on the correctly predicted $l_q$ subsets. To accomplish this, we design a KL-divergence (KL) loss between the predictions from $l_q$ and $l_r$ perturbations. For each batch of data $(x,y) \sim \mathcal{D}$, we generate $l_q$ and $l_r$ adversarial examples $x^\prime_{q}, x^\prime_{r}$ and their predictions $p_{q}, p_{r}$ using APGD~\citep{croce2020aa}. Then, we select indices $\gamma$, which part elements of $p_{q}$ correctly predicts the ground truth $y$. We denote the size of the indices as $n_{c}$, and the batch size as $N$. We compute a KL-divergence loss over this set of samples using $KL(p_{q}[\gamma] \| p_{r}[\gamma])$ (Eq.~\ref{eq:5}). For the subset indexed by $\gamma$, we want to push its $l_r$ logit distribution towards its $l_q$ logit distribution, such that we prevent losing more $l_q$ robustness when training with $l_r$ adversarial examples.

\begin{equation} \label{eq:5}
 \mathcal{L}_{KL} = \frac{1}{n_c} \cdot \sum^{n_c}_{i=1} \sum^{k}_{j=0}  p_{q}[\gamma[i]][j] \cdot \log \left( \frac{p_{q}[\gamma[i]][j]}{p_{r}[\gamma[i]][j]}\right)
\end{equation}

%We note that $\mathcal{L}_{KL}$ alone does not improve union accuracy when the $l_q$ robustness is low.
To further boost the union accuracy, apart from the KL loss, we add another loss term using a MAX-style approach in Eq.~\ref{eq:6}: we find the worst-case example between $l_q$ and $l_r$ adversarial regions by selecting the example with the higher loss. $\mathcal{L}_{max}$ is a cross-entropy loss over the approximated worst-case adversarial examples. Here, we use $\mathcal{L}_{ce}$ to represent the cross-entropy loss. Our final loss $\mathcal{L}$ combines $\mathcal{L}_{KL}$ and $\mathcal{L}_{max}$, via a hyper-parameter $\lambda$  in Eq.~\ref{eq:7}.

\begin{minipage}{.60\linewidth}
\begin{equation} \label{eq:6}
 \mathcal{L}_{max} = \frac{1}{N} \sum^{N}_{i=0} \left[\max_{p \in \{q,r\}} \max_{x_i^{\prime} \in B_p (x, \epsilon_p)} \mathcal{L}_{ce}(f(x_{i}^{\prime}),y_i)\right]
\end{equation}
\end{minipage}%
\begin{minipage}{.40\linewidth}
\begin{equation} \label{eq:7}
 \mathcal{L} = \mathcal{L}_{max} + \lambda \cdot \mathcal{L}_{KL}
\end{equation}
\end{minipage}

Algorithm~\ref{alg:robust-finetune} shows the pseudocode of robust fine-tuning with \textbf{RAMP} that leverages logit pairing. 

% This procedure can be easily adapted to the robust fine-tuning of many $l_{\infty}$ models, where we show their superior union performances in Section~\ref{sec:robust-finetune}.
\subsection{Connecting Natural Training with AT}~\label{nt-at}
To improve the robustness and accuracy tradeoff against multiple perturbations, we explore the connections between AT and NT. Since extracting valuable information in NT aids in improving robustness (Section~\ref{sec:nt-can-help-adv}), we use gradient projection~\citep{jiang2023fedgp} to compare and integrate natural and adversarial model updates, which yields an improved tradeoff between robustness and accuracy.

% where we discover there exists some useful information in NT that helps improve robustness. To extract the useful components from NT and incorporate them into the AT procedure, we leverage gradient projection~\citep{jiang2023fedgp} to compare and combine the natural and adversarial model updates, where we manage to obtain a better robustness and accuracy tradeoff. 

\noindent\textbf{NT can help adversarial robustness.}\label{sec:nt-can-help-adv}
%% In this section, we provide some insights and preliminary findings on the usefulness of pre-training. 
Let us consider two models $f_1$ and $f_2$, where $f_1$ is randomly initialized and $f_2$ undergoes NT on $\widehat \gD_n$ for $k$ epochs: $f_2$ results in a better \emph{decision boundary} and higher clean accuracy. Performing AT on $f_1$ and $f_2$ subsequently, intuitively, $f_2$ becomes more robust than $f_1$ due to its improved decision boundary, leading to fewer misclassifications of adversarial examples. This effect is empirically shown in Figure~\ref{fig:at-gp-effect}. For \textbf{AT} (blue), standard AT against $l_\infty$ attack~\citep{madry2017towards} is performed, while for \textbf{AT-pre} (red), $50$ epochs of pre-training precede the standard AT procedure. \textbf{AT-pre} shows superior clean and robust accuracy on CIFAR-10 against $l_\infty$ PGD-20 attack with $\epsilon_{\infty} = 0.031$. Despite $\widehat \gD_n$ and $\widehat \gD_a$ are different, Figure~\ref{fig:at-gp-effect} suggests valuable information in $\widehat \gD_n$ that potentially enhances performance on $\widehat \gD_a$.

\noindent\textbf{AT with Gradient Projection.}
To connect NT with AT more effectively, we analyze the training procedures on $\widehat \gD_n$ and $\widehat \gD_a$. We consider model updates over all samples from $\widehat \gD_n$ and $\widehat \gD_a$, with the initial model $f^{(r)}$ at epoch $r$, and models $f_{n}^{(r)}$ and $f_{a}^{(r)}$ after $1$ epoch of natural and adversarial training from the same starting point $f^{(r)}$, respectively. Here, we compare the natural updates $\widehat g_{n} = f_{n}^{(r)} - f^{(r)}$ and adversarial updates $\widehat g_{a} = f_{a}^{(r)} - f^{(r)}$. Due to distribution shift, an \emph{angle} exists between them. Our goal is to identify useful components from $g_{n}$ and incorporate them into $g_{a}$ for increased robustness in $\widehat \gD_a$ while maintaining accuracy in $\widehat \gD_n$. Inspired by~\citet{jiang2023fedgp}, we \emph{layer-wisely} compute the cosine similarity between $\widehat g_{n}$ and $\widehat g_{a}$. For a specific layer $l$ of $\widehat g_{n}^{l}$ and $\widehat g_{a}^{l}$, we preserve a portion of $\widehat g_{n}^{l}$ based on their cosine similarity score (Eq.\ref{eq:1}). Negative scores indicate that $\widehat g_{n}^{l}$ is not beneficial for robustness in $\widehat \gD_a$. Therefore, we filter components with similarity score $\leq 0$. We define the \textbf{GP} (Gradient Projection) operation in Eq.\ref{eq:2} by projecting $\widehat g_{a}^{l}$ towards $\widehat g_{n}^{l}$. 

% If the score is high, we preserve a larger portion of it.
% we take a closer look at the training procedure on $\widehat \gD_n$ and $\widehat \gD_a$. We consider the model updates over all samples from $\widehat \gD_n$ and $\widehat \gD_a$, where we define the initial model $f^{(r)}$ at epoch $r$, and the models $f_{n}^{(r)}$, and $f_{a}^{(r)}$ after $1$ epoch of natural training and adversarial training from the same starting point $f^{(r)}$, respectively. Here, we compare the natural updates $g_{n} = f_{n}^{(r)} - f^{(r)}$ and adversarial updates $g_{a} = f_{a}^{(r)} - f^{(r)}$: they are different because of the distribution shift, and there exists an \emph{angle} between them. We want to find the useful components from $g_{n}$ and combine them into $g_{a}$, so we can gain more robustness in $\widehat \gD_a$ as well as preserve accuracy in $\widehat \gD_n$. Inspired by~\cite{jiang2023fedgp}, we compare the model updates \emph{layer-wise} by computing the cosine similarity between $g_{n}$ and $g_{a}$. For a certain layer $l$ of $g_{n}^{l}$ and $g_{a}^{l}$, we preserve a certain portion of $g_{n}^{l}$ using their cosine similarity score (Eq.~\ref{eq:1}). If the score is negative, it means $g_{n}^{l}$ is not helpful for robustness in $\widehat \gD_a$, thus we filter out the useless components that have a similarity score $\leq 0$. If the score is high, we want to preserve a larger portion of it. We define \textbf{GP} (Gradient Projection) operation in Eq.~\ref{eq:2} by projecting $g_{a}^{l}$ towards $g_{n}^{l}$. 

\begin{minipage}{.35\linewidth}
  \begin{equation}\label{eq:1} 
    \cos (\widehat g_{n}^l, \widehat g_{a}^l) =  {\widehat g_{n}^l \cdot \widehat g_{a}^l \over \|\widehat g_{n}^l\| \|\widehat g_{a}^l\|}  
\end{equation}
\end{minipage}%
\begin{minipage}{.65\linewidth}
  \begin{equation}\label{eq:2}
    \small
    \mathbf{GP} (\widehat g_{n}^l,\widehat g_{a}^l)=
    \begin{cases}
         \cos (\widehat g_{n}^l,\widehat g_{a}^l) \cdot \widehat g_{n}^l,  &\cos (\widehat g_{n}^l, \widehat g_{a}^l) > 0 \\
        0, &\cos (\widehat g_{n}^l, \widehat g_{a}^l) \leq 0
    \end{cases}
\end{equation}
\end{minipage}

Therefore, the total projected (useful) model updates $g_p$ coming from $\widehat g_{n}$ could be computed as Eq.~\ref{eq:3}. We use $\mathcal{M}$ to denote all layers of the current model update. Note that $\bigcup_{l \in \mathcal{M}}$ concatenates all layers' useful natural model update components. A hyper-parameter $\beta$ is used to balance the contributions of $g_{GP}$ and $\widehat g_{a}$, as shown in Eq.~\ref{eq:4}. By finding a proper $\beta$ (0.5 as in Figure~\ref{fig:ablation-beta}), we can obtain better robustness on $\widehat \gD_a$, as shown in Figure~\ref{fig:at-gp-effect} and Figure~\ref{atgp}. In Figure~\ref{fig:at-gp-effect}, with $\beta=0.5$, \textbf{AT-GP} refers to AT with GP; for \textbf{AT-GP-pre}, we perform $50$ epochs of NT before doing \textbf{AT-GP}. We see \textbf{AT-GP} obtains a better accuracy/robustness tradeoff than \textbf{AT}. We observe a similar trend for \textbf{AT-GP-pre} vs. \textbf{AT-pre}. Further, in Figure~\ref{atgp}, \textbf{RN-18 $l_\infty$-GP} achieves good clean accuracy and better robustness than \textbf{RN-18 $l_\infty$} against AutoAttack~\citep{croce2020aa}.
% with $\epsilon_\infty=\frac{8}{255}$.

\begin{minipage}{.35\linewidth}
\begin{equation}\label{eq:3}
% \small
   g_{p} =  \bigcup_{l \in \mathcal{M}} \mathbf{GP}(\widehat g_{n}^l,\widehat g_{a}^l)
\end{equation}
\end{minipage}%
\begin{minipage}{.65\linewidth}
\begin{equation} \label{eq:4}
% \small
f^{(r+1)} = f^{(r)} + \beta \cdot g_{p} + (1 - \beta) \cdot \widehat g_{a}
\end{equation} 
\end{minipage}

% We achieve a good clean accuracy by preserving model updates in $\widehat \gD_n$.

\begin{minipage}[t]{0.48\linewidth}
\begin{algorithm}[H]
   \caption{Fine-tuning via Logit Pairing
   }\label{alg:robust-finetune}
   % \fontsize{7.8}{10}\selectfont
   % \small
\begin{algorithmic}[1]
    \STATE \textbf{Input}: model $f$, input samples $(x, y)$ from distribution $\widehat \gD_n$, fine-tuning rounds $R$, hyper-parameter $\lambda$, adversarial regions $B_q, B_r$ with size $\epsilon_q$ and $\epsilon_r$, \textbf{APGD} attack.\\
        \FOR{$r=1, 2,..., R$}
        \FOR{$(x ,y) \sim \text{ training set } \mathcal{D}$}
        % \STATE // Generate $l_1$ and $l_{\infty}$ examples
        \STATE $x^\prime_{q}, p_{q} \gets \textbf{APGD}(B_q(x, \epsilon_q), y)$
        \STATE $x^\prime_{r}, p_{r} \gets \textbf{APGD}(B_r(x, \epsilon_r), y)$
        % // Generate $l_1$ and $l_{\infty}$ predicted logits
        % \STATE $ \gets f(x^\prime_{l_q}),  \gets f(x^\prime_{l_r})$
        \STATE $\gamma \gets where(argmax\text{ }p_{q} =y)$
        \STATE $n_c \gets \gamma.size()$
        \STATE calculate $\mathcal{L}$ using Eq.~\ref{eq:7} and update $f$
        \ENDFOR
        \ENDFOR
        \STATE \textbf{Output}: model $f$. 
\end{algorithmic}
\end{algorithm}
\end{minipage}\hfill
\begin{minipage}[t]{0.45\linewidth}
\begin{algorithm}[H]
    \caption{Connect AT with NT via GP}
    % \fontsize{7.8}{10}\selectfont
    % \small
    \begin{algorithmic}[1]
        \STATE \textbf{Input}: model $f$, input images with distribution $\widehat \gD_n$, training rounds $R$, adversarial region $B_p$ and its size $\epsilon_p$, $\beta$, natural training \textbf{NT} and adversarial training \textbf{AT}.\\
        \FOR{$r=1, 2,..., R$}
        \STATE $f_{n} \gets \textbf{NT}(f^{(r)}, \mathcal{D})$  
        \STATE $f_{a} \gets \textbf{AT}(f^{(r)}, B_p, \epsilon_p, \mathcal{D})$
        \STATE compute $\widehat g_{n} \gets f_{n} - f^{(r)}$, $\widehat g_{a} \gets f_{a} - f^{(r)}$
        \STATE compute $g_p$ using Eq.~\ref{eq:3}
        \STATE update $f^{(r+1)}$ using Eq.~\ref{eq:4} with $\beta$ and $\widehat g_{a}$
        \ENDFOR
        \STATE \textbf{Output}: model $f$. 
    \end{algorithmic}
    \label{algo:atgp}
\end{algorithm}
\end{minipage}

\subsection{Theoretical Analysis of GP for Adversarial Robustness} We define \( \gD_n \) $=\{(x_i, y_i)\}_{i=0}^{\infty}$ as the ideal data distribution with an infinite cardinality. Here, we consider a classifier $f_\theta$ at epoch $t$. We define $\gD_a$ as the distribution created by \(\{ (x_i + \mathcal{\epsilon}(f_\theta, x_i, y_i), y_i)\}^{\infty}_{i=0}\) where $(x_i,y_i) \sim \gD_n$. \( x_i + \mathcal{\epsilon}(f_\theta, x_i, y_i) \) denotes the perturbed image, which could be both single and multiple perturbations based on \( f_\theta \) itself.

\begin{assumption} We assume $\widehat\gD_n$ consists of $N$ i.i.d. samples from the ideal distribution $\gD_n$ and $\widehat\gD_a=\{(x_i + \epsilon(f^\theta, x_i, y_i), y_i)\}^{N}_{i=0}$ where $(x_i,y_i) \sim \widehat\gD_n$ consists of $N$ i.i.d. samples from $\gD_a$. 
% \(\widehat\gD_a\) is created by \(\{(x + \epsilon(f^\theta, x), y)\}^{N}_{i=0}\) where $(x,y) \sim \widehat\gD_n$. 
\end{assumption}

% For a classifier $f^\theta$ at certain epoch $t$, we have its $\widehat \gD_a = \{(x + \epsilon(f^\theta, x), y)\}^{N}_{i=0}$. 
We define the population loss as $\mathcal{L}_\gD(\theta):=\E_{(x,y)\sim \gD}\mathcal{L}(f(x), y)$, and let $g_\gD(\theta):=\nabla\mathcal{L}_\gD (\theta).$ For simplification, we use $g_a :=\nabla\mathcal{L}_{\gD_a} (\theta)$, $\widehat g_a:=\nabla\mathcal{L}_{\widehat\gD_a} (\theta)$, and $\widehat g_n :=\nabla\mathcal{L}_{\widehat\gD_n} (\theta)$. $g_{GP} = \beta \cdot g_{p} + (1 - \beta) \cdot \widehat g_{a}$ (Definition~\ref{def:gp_aggr}) is the aggregation using GP. We define the following optimization problem.

% Note that this implies $\E_{\widehat\gD_a}[ g_{\widehat\gD_a}]=g_{\gD_a}$. 

% and \(\widehat\gD_a\) created by \(\{(x + \epsilon(f_{\theta}, x), y)\}^{N}_{i=0}\) where $(x,y) \sim \widehat\gD_n$.
% For the problem, we seek an aggregation strategy $\texttt{Aggr}(\widehat{g_a}, \widehat g_n)$ such that after training, the parameter $\theta$ minimizes the $\gD_a$ population loss function  $\ell_{\gD_a}(\theta)$.

% where $\E_{(x,y)\sim \gD}$ is the expectation w.r.t. $\gD$. Let $\widehat \gD$ be a finite sample dataset drawn from $\gD$, then $\ell_{\widehat\gD}(\theta):=\tfrac{1}{|\widehat \gD|}\sum_{z\in \widehat \gD}\ell(\theta, z)$, where $|\widehat \gD|=n$ is the size of the dataset. 
% We use $[N]:=\{1, 2, \dots, N\}$. By default, $\langle \cdot, \cdot\rangle$, and  $\|\cdot\|$ denote the Euclidean inner product and Euclidean norm, respectively.
% We define the model updates per epoch on these distributions as $g_a, \widehat{g_a}, \widehat g_n$, respectively; 
\begin{definition}[Aggregation for NT and AT]\label{def:aggr} $f_{\theta}$ is trained by iteratively updating the parameter
\begin{align}
    %\text{iteratively updating model parameter  }&\quad 
    \theta \leftarrow \theta - \mu \cdot \texttt{Aggr}(\widehat g_a, \widehat g_n),
\end{align}
where $\mu$ is the step size. We seek an aggregation rule $\texttt{Aggr}(\cdot) = \widehat g_{\texttt{Aggr}}$ such that after training, $f_{\theta}$ minimizes the population loss function  $\mathcal{L}_{\gD_a}(\theta)$. 
% Note that we allow the aggregation function $\texttt{Aggr}(\cdot)$ to depend on the iteration index. 
\end{definition}

We need $\widehat g_{\texttt{Aggr}}$ to be close to $g_a$ for each iteration, since $g_a$ is the optimal update on $\gD_a$. Thus, we define $L^\pi$-Norm and delta error to indicate the performance of different aggregation rules.

\begin{definition}[$L^\pi$-Norm~\citep{anonymous2024principled}]
Given a distribution $\pi$ on the parameter space $\theta$, we define an inner product $\langle g_\gD, g_{\gD'}\rangle_\pi=\E_{\theta\sim \pi}[\langle g_\gD(\theta),g_{\gD'}(\theta)\rangle]$. The inner product induces the $L^\pi$-norm on $g_\gD$ as $\|g_\gD\|_\pi:=\sqrt {\E_{\theta\sim \pi}\|g_\gD(\theta)\|^2}$. We use $L^\pi$-norm to measure the gradient differences under certain $\gD$.

% With the $L^\pi$-norm, we define the $L^\pi$ space as $\{g:\Theta\to \Theta \mid \|g\|_\pi<\infty\}$. 
% The inner product induces the $L^\pi$-norm on
\end{definition}

\begin{definition}[Delta Error of an aggregation rule \texttt{Aggr}$(\cdot)$] \label{def:delta-error} We define the following squared error term to measure the closeness between $\widehat g_{\texttt{Aggr}}$ and $g_a$ under $\widehat\gD_a^t$ (distribution at time step $t$), i.e.,
    \begin{align}
        \Delta^2_{\texttt{Aggr}}:= \E_{\widehat\gD_a^t}\|g_{a}-\widehat g_{\texttt{Aggr}}\|^2_\pi.
    \end{align}
%    where $\|\cdot\|_{\pi}$ is the $\pi$-norm over the model parameter space. We use $\pi$ to measure the gradient difference in the parameter space.
\end{definition}

% $\gD_a$ as the distribution created by adversarial examples generated on a perfect classifier, trained on an example set of size $\rightarrow \infty$. $\widehat\gD_a$ is the distribution created by adversarial examples generated by the AT classifier using a limited number of clean images. $\widehat\gD_n$ is the distribution of original clean images. 
 Delta errors $\Delta^2_{{AT}}$ and $\Delta^2_{{GP}}$ measure the closesness of $g_{GP}, \widehat g_a$ from $g_a$ in $\widehat\gD_a$ at each iteration. %(where $g_{\texttt{Aggr}} = \widehat g_a$ or $g_{\texttt{Aggr}} = g_{GP}$). 

 \begin{theorem}[Error Analysis of GP]\label{thm:error-GP}
When the model dimension $m \to \infty$, for an epoch $t$, we have an approximation of the error difference $\Delta^2_{{AT}} - \Delta^2_{{GP}}$ as follows
\begin{align} 
    \Delta^2_{AT} - \Delta^2_{GP} \approx  \beta (2 - \beta) \E_{\widehat\gD_a^t} \|g_{a} - \widehat{g_{a}} \|^2_\pi - \beta^2 \bar\tau^2 \|g_{a} -  \widehat g_{n}\|^2_\pi\label{eq:thm-fedgp}
\end{align}
% is the Delta error when only considering $\gD_{S_i}$ as the source domain. $m$ is the model dimension and In the above equation, 
% where $\|\cdot\|_{\pi}$ is the $\pi$-norm over the model parameter space. 
% %Moreover, the inequality~\eqref{eq:thm-fedgp} becomes equality when $N=1$.
$\bar \tau^2= \E_\pi[\tau^2] \in [0, 1]$, where $\tau (\theta)$ is the $\sin(\cdot)$ value of the angle between $\widehat g_{n}$ and $g_{a}- \widehat g_{n}$. 
\end{theorem}
 
Theorem~\ref{thm:error-GP} shows $\Delta^2_{{GP}}$ is generally smaller than $\Delta^2_{{AT}}$ for a large model dimension during each iteration, as is the case for the models in our evaluation, with $\beta = 0.5$, since $\beta(1-\beta) > \beta^2 (0.75 > 0.25)$ and the small value of $\bar\tau$ in practice (see Interpretation of Theorem~\ref{thm:error-diff} in Appendix~\ref{proof-theory}, where we show the order of difference is between $1e^{-8}$ and $1e^{-12}$). Thus, GP achieves better robust accuracy than AT by achieving a smaller delta error; GP also obtains good clean accuracy by combining parts of the model updates from the clean distribution $\widehat \gD_n$. Further, we provide an error analysis of a single gradient step in Theorem~\ref{thm-converge-gen} and convergence analysis in Theorem~\ref{thm:convergence}, showing that a smaller Delta error results in better convergence. The full proof of all theorems is in Appendix~\ref{proof-theory}. 

 %  (at time $t^\prime$, we also have $\Delta^2_{{GP}} \leq \Delta^2_{{AT}}$)

\begin{figure}[h!]
\vspace*{-\baselineskip}
\begin{minipage}[t]{.60\linewidth}
  \begin{subfigure}[b!]{.5\linewidth}
    \centering
    \includegraphics[width=\linewidth]{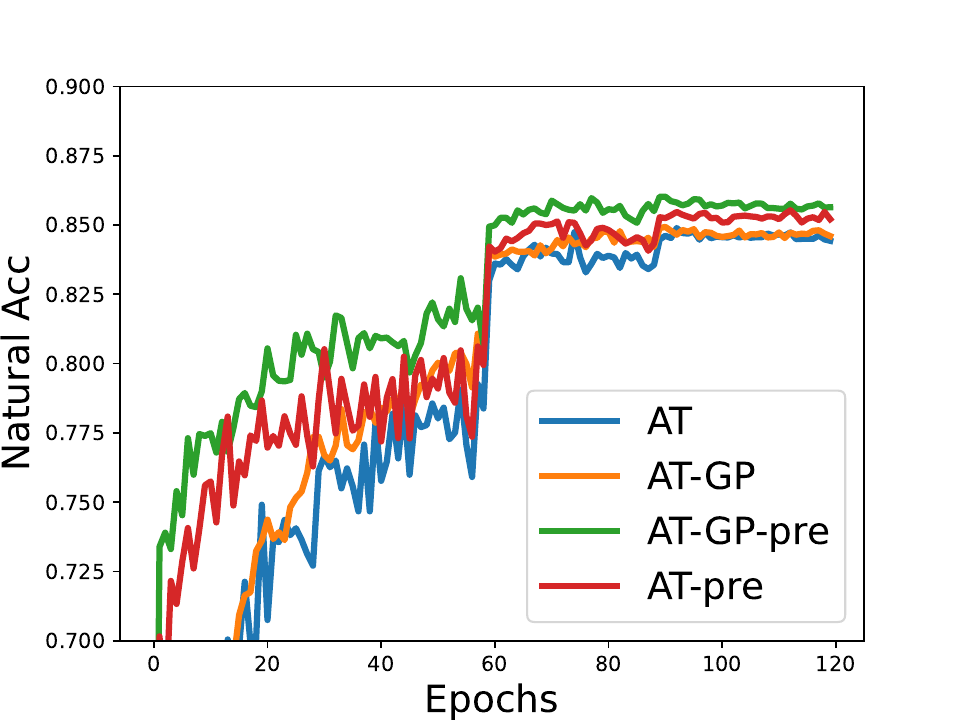}
    \caption{Clean Accuracy}
    \label{fig-at-gp-clean}
  \end{subfigure}\hfill
  \begin{subfigure}[b!]{.5\linewidth}
    \centering
    \includegraphics[width=\linewidth]{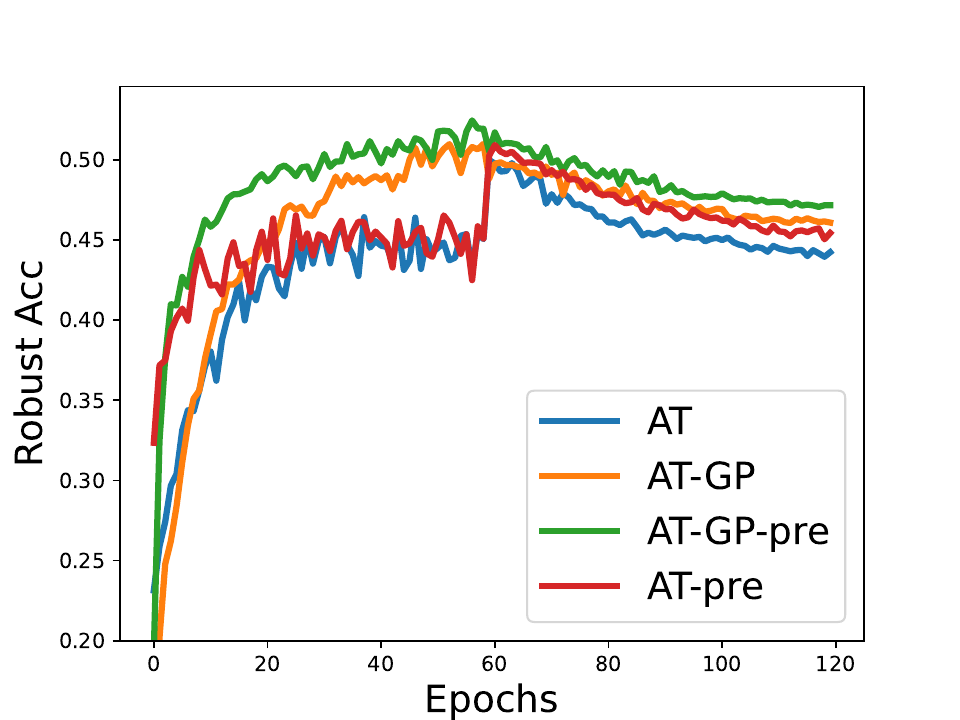}
    \caption{Robust Accuracy: PGD-20}
    \label{fig-at-gp-robust}
\end{subfigure}
  \caption{$l_\infty$ \textbf{AT-GP} with PGD~\citep{madry2017towards} with $\epsilon=0.031$ on CIFAR-10 improves accuracy and robustness. Pre-training on $\widehat \gD_n$ for $50$ epochs further boosts the performance.} 
  \label{fig:at-gp-effect}
\end{minipage}\hfill
\begin{minipage}[t]{.38\linewidth}
% \begin{table}[!h]
\centering
% \small
\fontsize{8.5}{10}\selectfont
% \tiny
\begin{tabular}{lrr}
                  & Clean & $l_\infty$ \\\hline
RN-18 $l_\infty$        & $84.2$                   & $47.4$                                \\
RN-18 $l_\infty$-GP     & $84.5$                   & {$\bf48.3$}                      \\
RN-18 $l_\infty$-GP-pre & {$\bf84.9$}          & {$\bf48.3$}              
\\\hline
\end{tabular}
% \end{table}
\caption{$l_\infty$ \textbf{AT-GP} with APGD~\citep{croce2020aa} improves robustness against $l_\infty$ AutoAttack~\citep{croce2020aa} with $\epsilon = \frac{8}{255}$. RN-18 $l_\infty$-GP uses \textbf{AT-GP}; RN-18 $l_\infty$-GP-pre pre-trains $40$ epochs on $\widehat \gD_n$ before \textbf{AT-GP} is applied.}\label{atgp}
\end{minipage}
\vspace*{-\baselineskip}
\end{figure}

We outline the \textbf{AT-GP} method in Algorithm~\ref{algo:atgp} and it can be extended to the multiple-norm scenario. The overhead of this algorithm comes from natural training and GP operation. Their costs are small, and we discuss this more in Section~\ref{sec:ablation}. Combining logit pairing and gradient projection methods, we provide the \textbf{RAMP} framework which is similar to Algorithm~\ref{algo:atgp}, except that we replace line 4 of Algorithm~\ref{algo:atgp} as Algorithm~\ref{alg:robust-finetune} line 3-9. 

% We show that DNNs trained or fine-tuned using \textbf{RAMP} achieve SOTA union accuracy in Section~\ref{sec:robust-finetune}.

% \begin{algorithm}[h!]
%     \caption{Adversarial Training with Gradient Projection}
%     \begin{algorithmic}[1]
%         \STATE \textbf{Input}: model $f$, input images with distribution $\widehat \gD_n$, training rounds $R$, adversarial region $B_p$ and its size $\epsilon_p$, $\beta$, natural training \textbf{NT} and adversarial training \textbf{AT}.\\
%         \FOR{$r=1, 2,..., R$}
%         \STATE $f_{n} \gets \textbf{NT}(f^{(r)}, \mathcal{D})$  
%         \STATE $f_{a} \gets \textbf{AT}(f^{(r)}, B_p, \epsilon_p, \mathcal{D})$
%         \STATE compute $g_{n} \gets f_{n} - f^{(r)}$, $g_{a} \gets f_{a} - f^{(r)}$
%         \STATE compute $g_{GP}$ using Eq.~\ref{eq:3}
%         \STATE update $f^{(r+1)}$ using Eq.~\ref{eq:4} with $\beta$ and $g_{a}$
%         \ENDFOR
%         \STATE \textbf{Output}: model $f$. 
%     \end{algorithmic}
%     \label{algo:atgp}
% \end{algorithm}

% \subsection{RAMP Algorithm}\label{sec:ramp}

% the standard AT procedure of getting $f_{adv}$
\section{Experiment}
% \begin{table*}[!h]
%     \centering
%     % \small
%     \fontsize{8}{10}\selectfont
%     % \fontsize{7.8}{10}\selectfont
%     \begin{tabular}{lrrrrr}
% \text {Methods} & \text { Clean } & $l_{\infty}$ & $l_2$ & $l_1$ & \text { Union }\\
% \hline RN-18- $l_{\infty}$-AT & $83.7$ & $48.1$ & $59.8$ & $7.7$ & $38.5$ \\
% + SAT & $83.5 \pm 0.2$ & $43.5 \pm 0.2$ & $68.0 \pm 0.4$ & $47.4 \pm 0.5$ & $41.0 \pm 0.3$ \\
% + AVG & $84.2 \pm 0.4$ & $43.3 \pm 0.4$ & $68.4 \pm 0.6$ & $46.9 \pm 0.6$ & $40.6 \pm 0.4$\\
% + MAX & $82.2 \pm 0.3$ & $45.2 \pm 0.4$ & $67.0 \pm 0.7$ & $46.1 \pm 0.4$ & $42.2 \pm 0.6$\\
% + MSD & $82.2 \pm 0.4$ & $44.9 \pm 0.3$ & $67.1 \pm 0.6$ & $47.2 \pm 0.6$ & $42.6 \pm 0.2$ \\
% + E-AT & $82.7 \pm 0.4$ & $44.3 \pm 0.6$ & $68.1 \pm 0.5$ & $48.7 \pm 0.5$ & $42.2 \pm 0.8$ \\
% + \textbf{RAMP} & $81.1 \pm 0.2$ & $45.4 \pm 0.3$ & $66.1 \pm 0.2$ & $47.2 \pm 0.1$ & $\mathbf{43.1 \pm 0.2}$ \\\hline
% \end{tabular}
%     \caption{\textbf{RN-18 $l_\infty$-AT model fine-tuned} for 3 epochs (repeated for 5 seeds). \textbf{RAMP} has the highest union accuracy. Baseline results are from~\cite{croce2022eat}.}
%     % Time per epoch is evaluated using a single Tesla V100 GPU.
%     \label{table:resnet18-robust-finetune}
% \end{table*}

% Additional results exploring fine-tuning epochs and logit pairing losses are in Appendix~\ref{supp}.

% \subsection{Experimental Setup}

\noindent \textbf{Datasets, baselines, and models.} CIFAR-10~\citep{krizhevsky2009cifar10} includes $60$K images with $50$K and $10$K images for training and testing respectively. ImageNet has $\approx 14.2$M images and $1$K classes, containing $\approx 1.3$M training, $50$K validation, and $100$K test images~\citep{russakovsky2015imagenet}. We compare \textbf{RAMP} with following baselines: 1. \textbf{SAT}~\citep{madaan2021sat}: randomly sample one of the $l_1$, $l_2$, $l_\infty$ attacks. 2. \textbf{AVG}~\citep{tramer2019atmp}: take the average of $l_1, l_2, l_\infty$ examples. 3. \textbf{MAX}~\citep{tramer2019atmp}: take the worst of $l_1, l_2, l_\infty$ attacks. 4. \textbf{MSD}~\citep{maini2020msd}: find the worst-case examples over $l_1, l_2, l_\infty$ steepest descent directions during each step of inner maximization. 5. \textbf{E-AT}~\citep{croce2022eat}: randomly sample between $l_1$, $l_\infty$ attacks. For models, we use PreAct-ResNet-18, ResNet-50, WideResNet-34-20, and WideResNet-70-16 for CIFAR-10, as well as ResNet-50 and XCiT-S transformer for ImageNet.
% 6. \textbf{$\{l_1, l_2, l_\infty\}$-AT}: adversarially pre-trained model against $\{l_1, l_2, l_\infty\}$ attacks, respectively.

\noindent \textbf{Implementations and Evaluation.} For AT from scratch for CIFAR-10, we train PreAct ResNet-18~\citep{he2016resnet} with a $lr=0.05$ for $70$ epochs and $0.005$ for $10$ more epochs. We set $\lambda=2$, $\beta=0.5$ for training from scratch, and $\lambda=0.5$ for robust fine-tuning. For all methods, we use $10$ steps for the inner maximization in AT. For ImageNet, we perform $1$ epoch of fine-tuning and use a learning rate $lr = 0.005$, $\lambda=0.5$ for ResNet-50 and $lr = 1e^{-4}$, $\lambda=0.5$ for XCiT-S models. We reduce the rate by a factor of $10$ every $\frac{1}{3}$ of the training epoch and set the weight decay to $1e^{-4}$. We use APGD with $5$ steps for $l_\infty$ and $l_2$, $15$ steps for $l_1$. Settings are similar to~\citep{croce2022eat}. We use the standard values of $\epsilon_1=12, \epsilon_2=0.5, \epsilon_\infty=\frac{8}{255}$ for CIFAR-10 and $\epsilon_1=255, \epsilon_2=2, \epsilon_\infty=\frac{4}{255}$ for ImageNet. We focus on $l_\infty$-AT models for fine-tuning, as \citet{croce2022eat} shows their higher union accuracy for the $\epsilon$ values in our evaluation. We report the clean accuracy, robust accuracy against $\{l_1, l_2, l_\infty\}$ attacks, union accuracy, universal robustness against common corruptions and unseen adversaries, as well as runtime for \textbf{RAMP}. The robust accuracy is evaluated using Autoattack~\citep{croce2020aa}. More implementation details are in Appendix~\ref{supp}.

%Besides, we discover that most SOTA models are pre-trained on $l_\infty$ perturbations. 
%For robust fine-tuning, we perform $3$ epochs on CIFAR-10. We set the learning rate as $0.05$ for PreAct-ResNet-18 and $0.01$ for other models. We set $\lambda=1.5$ in this case. Also, we reduce the learning rate by a factor of $10$ after completing each epoch.

%with the same hyperparameters

% We also report the training time per epoch for AT with full training and robust fine-tuning.
% The main results are run with 5 trials with variances.

\subsection{Main Results}
% \vspace*{-0.5cm}
\begin{table}[!h]
    \caption{\textbf{Different epsilon values}: \textbf{RAMP} consistently outperforms E-AT and MAX for both training from scratch and robust fine-tuning when the key tradeoff pair changes.}
\centering
\fontsize{7.5}{8}\selectfont
% \tiny
\begin{tabular}{llrrrrr|rrrrr}
                                       & \multicolumn{1}{c}{} & \multicolumn{5}{c}{$(12, 0.5, \frac{2}{255})$}                               & \multicolumn{5}{c}{$(12, 1.5,\frac{8}{255})$}                      \\
                                       &                      & Clean & $l_\infty$ & $l_2$   & $l_1$            & Union         & Clean & $l_\infty$ & $l_2$   & $l_1$   & Union         \\\hline
\multirow{2}{*}{Training from Scratch} & E-AT                 & 87.2  & 73.3 & 64.1 & 55.4 & 55.4          & 83.5  & 41.0   & 25.5 & 52.9 & 25.5          \\
& MAX  & 85.6  & 72.1 & 63.6 & 56.4          & 56.4 & 74.6  & 42.9 & 35.7 & 50.3 & 35.6        \\
                                       & \textbf{RAMP}                 & 86.3  & 73.3 & 64.9 & 59.1 & \textbf{59.1} & 74.4  & 43.4 & 37.2 & 51.1 & \textbf{37.1} \\\hline
\multirow{2}{*}{Robust Fine-tuning}    & E-AT                 & 86.5  & 74.8 & 66.7 & 57.9          & 57.9          & 80.2  & 42.8 & 31.5 & 52.4 & 31.5          \\
& MAX  & 85.7  & 74.0   & 66.2 & 60.0            & 60.0            & 74.8  & 43.8 & 36.7 & 50.2 & 36.6          \\
                                       & \textbf{RAMP}               & 85.8  & 74.0 & 66.2 & 60.1 & \textbf{60.1} & 74.9  & 43.7 & 37.0 & 50.2 & \textbf{36.9}\\\hline  
\end{tabular}
\vspace*{-\baselineskip}
\label{table:more-eps-values}
\end{table}

\begin{table}[h!]
\vspace*{-\baselineskip}
\centering
\caption{\textbf{Robust fine-tuning on larger models and datasets} (* uses extra data for pre-training). We evaluate all CIFAR-10 and Imagenet test points. \textbf{RAMP} consistently achieves better union accuracy with significant margins and good accuracy-robustness tradeoff.}
% We fine-tune for 3 epochs on CIFAR-10 and 1 epoch on ImageNet.
\fontsize{6}{7}\selectfont
% \tiny
\begin{tabular}{lllrrrrr}
         & Models        & Methods       & \multicolumn{1}{l}{Clean} & \multicolumn{1}{l}{$l_\infty$} & \multicolumn{1}{l}{$l_2$} & \multicolumn{1}{l}{$l_1$} & \multicolumn{1}{l}{Union} \\\hline
         & WRN-70-16-$l_\infty$(*)~\citep{gowal2020uncovering} & E-AT  & 89.6  & 54.4 & 76.7 & 58.0 & 51.6          \\
         &               & \textbf{RAMP} & 90.6  & 54.7 & 74.6 & 57.9 & \textbf{53.3} \\
         & WRN-34-20-$l_\infty$~\citep{gowal2020uncovering}     & E-AT & 87.8  & 49.0 & 71.6 & 49.8 & 45.1          \\
         &               & \textbf{RAMP} & 87.1  & 49.7 & 70.8 & 50.4 & \textbf{46.9} \\
         & WRN-28-10-$l_\infty$(*)~\citep{carmon2019unlabeled} & E-AT & 89.3  & 51.8 & 74.6 & 53.3 & 47.9          \\
   \textbf{CIFAR-10}      &               & \textbf{RAMP} & 89.2  & 55.9 & 74.7 & 55.7 & \textbf{52.7} \\
         & WRN-28-10-$l_\infty$(*)~\citep{gowal2020uncovering} & E-AT & 89.8  & 54.4 & 76.1 & 56.0 & 50.5          \\
         &               & \textbf{RAMP} & 89.4  & 55.9 & 74.7 & 56.0 & \textbf{52.9} \\
          & RN-50-$l_\infty$~\citep{robustness}        & E-AT          & 85.3  & 46.5 & 68.3 & 45.3 & 41.6 \\
         &               & \textbf{RAMP} & 84.3  & 47.0 & 67.7 & 46.5 & \textbf{43.3} \\\hline
         & XCiT-S-$l_\infty$~\citep{debenedetti2022adversarially}        & E-AT & 68.4  & 38.1 & 51.8 & 23.8 & 23.4          \\
\textbf{ImageNet}&               & \textbf{RAMP} & 66.0  & 35.7 & 50.2 & 30.0   & \textbf{29.1}\\ 
& RN-50-$l_\infty$~\citep{robustness}         & E-AT & 58.2  & 26.9 & 39.5 & 18.8 & 17.8          \\
         &               & \textbf{RAMP}& 55.6  & 25.1 & 38.3 & 22.4 & \textbf{20.9} \\\hline
\end{tabular}
\vspace*{-\baselineskip}
\label{table:other-robust-finetune}
\end{table}

\noindent\textbf{Robust fine-tuning.}\label{sec:robust-finetune} In Table~\ref{table:other-robust-finetune}, we apply \textbf{RAMP} to larger models and datasets (ImageNet). However, the implementation of other baselines is not publicly available and~\citet{croce2022eat} do not report other baseline results except E-AT on larger models and datasets, so we only compare against E-AT in Table~\ref{table:other-robust-finetune}, which shows \textbf{RAMP} consistently obtains better union accuracy and accuracy-robustness tradeoff than E-AT. We observe that \textbf{RAMP} improves the performance more as the model becomes larger. We obtain the SOTA union accuracy of $53.3\%$ on CIFAR-10 and $29.1\%$ on ImageNet.

\noindent\textbf{RAMP with varying $\epsilon_1, \epsilon_2, \epsilon_\infty$ values.} 
% $\epsilon$'s in our paper are standard values used in previous work~\citep{croce2022eat}. 
We provide results with 1. $(\epsilon_1=12, \epsilon_2=0.5, \epsilon_\infty=\frac{2}{255})$ where $\epsilon_\infty$ size is small and 2. $(\epsilon_1=12, \epsilon_2=1.5, \epsilon_\infty=\frac{8}{255})$ where $\epsilon_2$ size is large, using PreAct ResNet-18 model for CIFAR-10 dataset: these cases have different tradeoff pair compared to  Figure~\ref{fig:l1linf-tradeoff-intro}. The pair identified using our heuristic are $l_1$ - $l_2$ and $l_2$ - $l_{\infty}$. In Table~\ref{table:more-eps-values}, we observe that \textbf{RAMP} consistently outperforms E-AT and MAX with significant margins in union accuracy, when training from scratch and performing robust fine-tuning. In Table~\ref{table:more-eps-values}, when $l_2$ is the bottleneck, E-AT obtains a lower union accuracy as it does not leverage $l_2$ examples. Similar observations are made across various epsilon values, with \textbf{RAMP} consistently outperforming other baselines, as detailed in Appendix~\ref{additional-diff-eps}. Appendix~\ref{supp} includes more training details/results, and ablation studies. Results for applying the trades loss to \textbf{RAMP} outperforming E-AT are detailed in Appendix~\ref{wideresnet-exp}. Appendix~\ref{robust-finetune-rn18-exp} presents robust fine-tuning using ResNet-18, where \textbf{RAMP} achieves the highest union accuracy. 

\begin{wraptable}{r}{0.60\linewidth}
\vspace*{-\baselineskip}
    \caption{\textbf{RN-18 model trained from random initialization} on CIFAR-10 over 5 trials: \textbf{RAMP} achieves the best union robustness and good clean accuracy compared with other baselines. Baseline results are from~\citet{croce2022eat}.}
    \centering
    \fontsize{6}{8}\selectfont
    \begin{tabular}
{@{}lrrrrr@{}}
 \text {Methods} & \text { Clean } & $l_{\infty}$ & $l_2$ & $l_1$ & \text { Union }\\\hline
% \hline $l_{\infty}$ \text {-AT } & 84.0 & 48.1 & 59.7 & 6.3 & 6.3 \\
% $l_2$ \text {-AT } & 88.9 & 27.3 & 68.7 & 25.3 & 20.9 \\
% $l_1$ \text {-AT } & 85.9 & 22.1 & 64.9 & 59.5 & 22.1 \\
% \hline
\text {SAT} & 83.9$\pm$0.8 & 40.7$\pm$0.7 & 68.0$\pm$0.4 & 54.0$\pm$1.2 & 40.4$\pm$0.7 \\
\text {AVG} & 84.6$\pm$0.3 & 40.8$\pm$0.7 & 68.4$\pm$0.7 & 52.1$\pm$0.4 & 40.1$\pm$0.8 \\
\text {MAX} & 80.4$\pm$0.5 & 45.7$\pm$0.9 & 66.0$\pm$0.4 & 48.6$\pm$0.8 & 44.0$\pm$0.7 \\
% \text { MSD }(*) & 82.1 & 43.1 & 64.5 & 46.5 & 41.4 \\
\text {MSD} & 81.1$\pm$1.1 & 44.9$\pm$0.6 & 65.9$\pm$0.6 & 49.5$\pm$1.2 & 43.9$\pm$0.8 \\
\text {E-AT} & 82.2$\pm$1.8 & 42.7$\pm$0.7 & 67.5$\pm$0.5 & 53.6$\pm$0.1 & 42.4$\pm$0.6 \\
\textbf {RAMP} ($\lambda$=5) & 81.2$\pm$0.3 & 46.0$\pm$0.5 & 65.8$\pm$0.2 & 48.3$\pm$0.6 & \textbf{44.6$\pm$0.6}\\
\textbf {RAMP} ($\lambda$=2) & 82.1$\pm$0.3 & 45.5$\pm$0.3 & 66.6$\pm$0.3 & 48.4$\pm$0.2 & 44.0$\pm$0.2\\\hline
\end{tabular}
\vspace*{-\baselineskip}
\label{table:resnet18-robust-pretrain}
% \end{table*}
\end{wraptable}
\noindent\textbf{Adversarial training from random initialization.}\label{sec:robust-pretrain} Table~\ref{table:resnet18-robust-pretrain} presents the results of AT from random initialization on CIFAR-10 with PreAct ResNet-18. \textbf{RAMP} has the highest union accuracy with good clean accuracy, which indicates that \textbf{RAMP} can mitigate the tradeoffs among perturbations and robustness/accuracy in this setting. The results for all baselines are from~\citet{croce2022eat}.

\begin{table}[h!]
\centering
% \vspace*{-1.0cm}
\fontsize{6}{8}\selectfont
\caption{Individual, average, and union accuracy against common corruptions (averaged across five levels) and unseen adversaries using WideResNet-28-10 on CIFAR-10 dataset.}
\begin{tabular}{lr|rrrrrrrr}
Models         & Common Corruptions & \multicolumn{1}{l}{$l_0$} & \multicolumn{1}{l}{fog} & \multicolumn{1}{l}{snow} & \multicolumn{1}{l}{gabor} & \multicolumn{1}{l}{elastic} & \multicolumn{1}{l}{jpeginf} & \multicolumn{1}{l}{Avg} & \multicolumn{1}{l}{Union}\\\hline
$l_1$-AT         & 78.2       & 79.0                                      & 41.4                              & 22.9                               & 40.5                               & 48.9                                    & 48.4                                & 46.9                    & 12.8                      \\
$l_2$-AT         & 77.2          & 67.5                                    & 48.7                              & 26.1                               & 44.1                               & 53.2                                    & 45.4                                & 47.5                    & 16.2                      \\
$l_\infty$-AT       & 73.4        & 55.5                                    & 44.7                              & 32.9                               & 53.8                               & 56.6                                    & 33.4                                & 46.2                    & 19.1                      \\
Winninghand~\citep{diffenderfer2021winning}   & \textbf{91.1}  & 74.1                           & 74.5                              & 18.3                               & 76.5                               & 12.6                                    & 0.0                                  & 42.7                    & 0.0                         \\
E-AT          & 71.5      & 58.5                                    & 35.9                              & 35.3                               & 50.7                               & 55.7                                    & 60.3                                & 49.4                    & 21.9                      \\
MAX           & 71.0         & 56.2                                    & 42.9                              & 35.4                               & 49.8                               & 57.8                                    & 55.7                                & 49.6                    & 24.4                      \\
\textbf{RAMP} & 75.5  & 55.5                                    & 40.5                              & 40.2                               & 52.9                               & 60.3                                    & 56.1                                & \textbf{50.9}           & \textbf{26.1}\\\hline                          
\end{tabular}\label{table:comm-corr-unseen-wide}
\vspace*{-\baselineskip}
\end{table}

\noindent\textbf{Universal Robustness.}\label{generality-main} In Table~\ref{table:comm-corr-unseen-wide}, we report average accuracy against common corruptions and union accuracy against unseen adversaries from~\citet{laidlaw2020perceptual} (implementation details are in Appendix~\ref{genality-other}). We compare against $l_p$ pretrained models, E-AT, MAX, winninghand~\citep{diffenderfer2021winning} (a SOTA method for natural corruptions) using WideResNet-28-10 architecture on the CIFAR-10 dataset. Compared to E-AT and MAX, \textbf{RAMP} achieves $4\%$ higher accuracy for common corruptions with five severity levels and $2$-$4\%$ better union accuracy against multiple unseen adversaries. Winninghand has high corruption robustness but no adversarial robustness. The results show that \textbf{RAMP} obtains a better robustness and accuracy tradeoff with stronger universal robustness. In Appendix~\ref{genality-other}, we evaluate on ResNet-18 to support this fact further.

% showcasing the superiority of \textbf{RAMP} over E-AT in WideResNet-28-10 for the CIFAR-10 dataset. 

% Further,  we see that identifying the tradeoff pair is important for higher union accuracy, e.g., as shown in Table~\ref{table:more-eps-values} when $l_2$ is the bottleneck, E-AT gets the lowest union accuracy as it does not leverage $l_2$ examples. We have similar observations for a range of other values of epsilons, where \textbf{RAMP} outperforms other baselines, as illustrated in Appendix~\ref{additional-diff-eps}.
%In Figure~\ref{fig:more-eps-values}, we show similar diagrams as Figure~\ref{fig:l1-linf-pretrain}. 

% \begin{table*}[t]

% \begin{table}[h!]
% \caption{\textbf{RAMP} with $l_\infty, l_1, l_2$-RN-18-AT models on CIFAR-10 with standard epsilons.}
% \centering
% % \small
% % \fontsize{8}{10}\selectfont
% \begin{tabular}{lrrrrr}
%                    & {Clean} & $l_\infty$ & $l_2$ & $l_1$ & Union \\\hline
% RN-18 $l_\infty$-AT & 80.9             & \textbf{45.5}            & 66.2        & \textbf{47.3}          & \textbf{43.1}             \\
% RN-18 $l_1$-AT & 78.0                      & 41.5                     & 63.4                   & 46.0                   & 40.4\\
% RN-18 $l_2$-AT &         \bf83.5              &     41.9             &    \bf 68.4         &    45.5         & 39.7\\\hline                     
% \end{tabular}
% \label{table:reverse-logit-pairing}
% \end{table}

% \begin{figure}
\begin{figure}
%\vspace*{-\baselineskip}
  \begin{subfigure}[b!]{0.33\linewidth}
  \vspace*{-\baselineskip}
    \centering
    \includegraphics[width=0.77\linewidth]{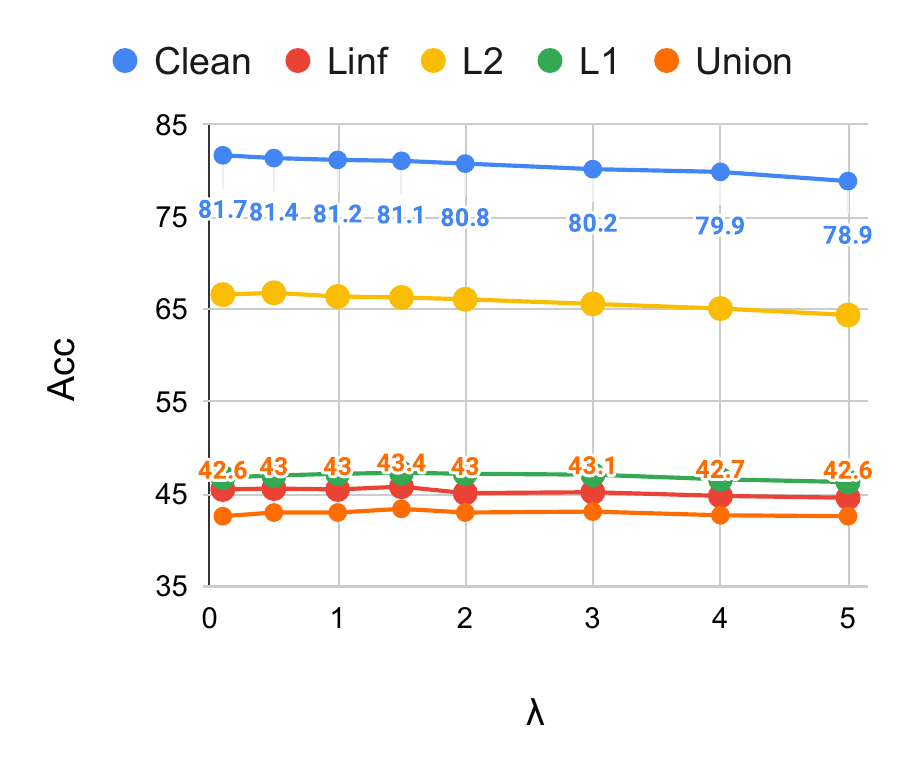}
    \caption{$\lambda$: Robust fine-tuning.}%Robust fine-tuning.
    \label{fig:ablation-lambda-finetune}
    \vspace*{-\baselineskip}
\end{subfigure}\hfill
  \begin{subfigure}[b!]{0.33\linewidth}
    \centering
    \includegraphics[width=0.77\linewidth]{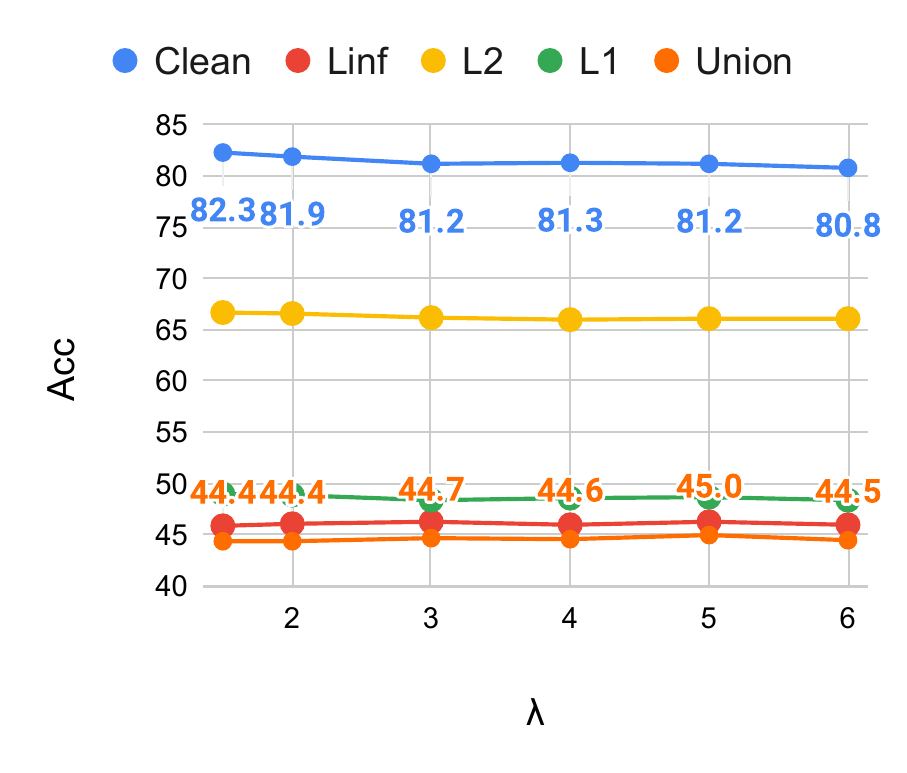}
    \caption{$\lambda$: Train from scratch.}%AT from random initializations.
    \label{fig:ablation-lambda-pretrain}
\end{subfigure}\hfill
\begin{subfigure}[b!]{0.33\linewidth}
    \centering
    \includegraphics[width=0.89\linewidth]{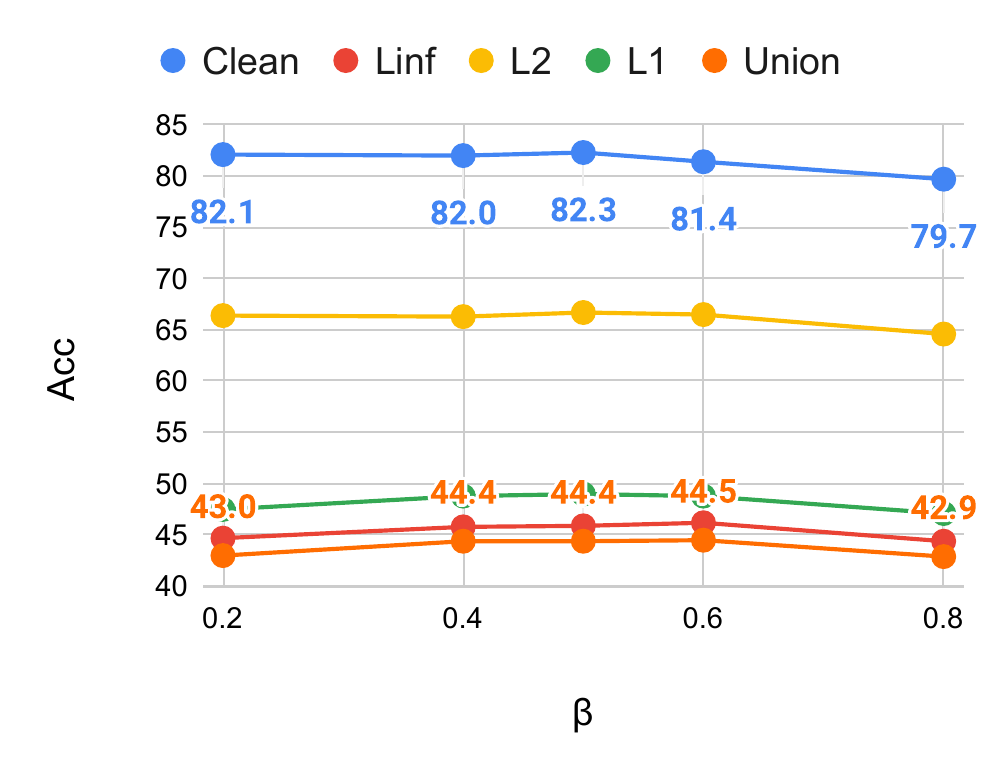}
    \caption{$\beta$: Train from scratch.}
    \label{fig:ablation-beta}
\end{subfigure}
  \caption{Alabtion studies on $\lambda$ and $\beta$ hyper-parameters.} 
  \label{fig:ablation-lambda}
\vspace*{-0.6cm}
\end{figure}
% \end{figure}

\subsection{Ablation Study and Discussion}\label{sec:ablation}
\noindent\textbf{Sensitivities of $\lambda$.} We perform experiments with different $\lambda$ values in $[0.1, 0.5, 1.0, 1.5, 2, 3, 4, 5]$ for robust fine-tuning and $[1.5, 2, 3, 4, 5, 6]$ for AT from scratch using PreAct-ResNet-18 model for CIFAR-10 dataset. In Figure~\ref{fig:ablation-lambda}, we observe a decreased clean accuracy when $\lambda$ becomes larger. We pick $\lambda=2.0$ for training from scratch (Figure~\ref{fig:ablation-lambda-finetune}) and $\lambda=0.5$ for robust fine-tuning (Figure~\ref{fig:ablation-lambda-pretrain}) in our main experiments, as these values of $\lambda$ yield both good clean and union accuracy.

\noindent\textbf{Choices of $\beta$.} Figure~\ref{fig:ablation-beta} shows the performance of \textbf{RAMP} with varying $\beta$ values on CIFAR-10 ResNet-18 experiments. We pick $\beta=0.5$ for combining natural training and AT via GP, which achieves comparatively good robustness and clean accuracy. This choice is also based on Theorem~\ref{thm:error-GP} when $\beta (2-\beta)$ has the largest difference from $\beta^2$ (0.75 vs 0.25). 

% For bigger $\beta$ values, we incorporate more adversarial training components, which leads to a lower clean accuracy. 

% \begin{figure}[t]
\begin{wraptable}{r}{0.45\linewidth}
% \vspace*{-0.6cm}
\centering
% \tiny
\fontsize{7}{9}\selectfont
\begin{tabular}{lrrrrr}
                   & {Clean} & $l_\infty$ & $l_2$ & $l_1$ & Union \\\hline
RN-18 $l_\infty$-AT & 81.5             & \textbf{45.5}            & 66.4        & \textbf{47.0}          & \textbf{42.9}             \\
RN-18 $l_1$-AT & 81.0                      & 42.6                     & 66.0                   & 48.1                   & 41.5\\
RN-18 $l_2$-AT &         \bf84.1              &     41.6             &    \bf 69.1         &    45.4         & 39.4\\\hline                     
\end{tabular}
\caption{\textbf{RAMP} with $l_\infty, l_1, l_2$-RN-18-AT models on CIFAR-10 with standard epsilons.}
% \vspace*{-\baselineskip}
\label{table:reverse-logit-pairing}
\end{wraptable}

\noindent\textbf{Fine-tune $l_p$ AT models with RAMP.} Table~\ref{table:reverse-logit-pairing} shows the robust fine-tuning results using \textbf{RAMP} with $l_\infty$-AT ($q=\infty, r=1$), $l_1$-AT ($q=1, r=\infty$), $l_2$-AT ($q=\infty, r=1$) RN-18 models for CIFAR-10 dataset. For $l_\infty - l_1$ tradeoffs, RAMP on $l_\infty$-AT pre-trained model achieves the best union accuracy. %than $l_1$-AT or $l_2$-AT model.

\noindent\textbf{Computational analysis and Limitations.} The extra training costs of AT-GP are small, e.g. for each epoch on ResNet-18, the extra NT takes $6$ seconds and the standard AT takes $78$ seconds using a single NVIDIA A100 GPU, and the \textbf{GP} operation only takes $0.04$ seconds on average. RAMP is more expensive than E-AT and less expensive than MAX. We have a complete runtime analysis in Appendix~\ref{time-analysis}. We notice occasional drops in clean accuracy during fine-tuning with \textbf{RAMP}. In some cases, union accuracy improves slightly but clean accuracy and single $l_p$ robustness reduce. Further, we find no negative societal impact from this work.

% \textbf{RAMP} has around $2$ times the runtime cost of E-AT and $\frac{2}{3}$ cost of MAX: \textbf{RAMP}, E-AT, and MAX take $157$, $78$, and $219$ seconds per epoch on CIFAR-10 with ResNet-18, respectively.

% Moreover, \textbf{RAMP} incurs higher time costs compared to defenses using random sampling.

% We observe that the clean accuracy sometimes drops when fine-tuning with \textbf{RAMP}. In some cases, union accuracy increases slightly at the cost of decreasing both clean accuracy and single $l_p$ robustness. %We hypothesize that \textbf{RAMP} improves union accuracy by finding stronger adversaries - the adversarial distribution $\mathcal{D}^\prime$ it created becomes further away from $\mathcal{D}$ with lower clean accuracy. 
% Also, \textbf{RAMP} has a higher time cost than defenses employing random sampling. In addition, we do not find any negative societal impact of this work.

\section{Conclusion}
We introduce \textbf{RAMP}, a framework enhancing multiple-norm robustness and achieving superior \emph{universal robustness} against corruptions and perturbations by addressing tradeoffs among $l_p$ perturbations and accuracy/robustness. We apply a new logit pairing loss and use gradient projection to obtain SOTA union accuracy with favorable accuracy/robustness tradeoffs against common corruptions and other unseen adversaries. Results demonstrate that \textbf{RAMP} surpasses SOTA methods in union accuracy across model architectures on CIFAR-10 and ImageNet.
% We present \textbf{RAMP}, a framework that boosts multiple-norm robustness, via alleviating the tradeoffs in robustness among multiple $l_p$ perturbations and accuracy/robustness. By analyzing the tradeoffs from the lens of distribution shifts, we identify the key tradeoff pair, apply logit pairing, and leverage gradient projection methods to boost union accuracy with good accuracy/robustness/efficiency tradeoffs. Our results show that \textbf{RAMP} outperforms SOTA methods with better union accuracy, on a wide range of model architectures on CIFAR-10 and ImageNet.

% \bibliographystyle{plain}
\bibliography{main}
\bibliographystyle{plainnat}

\appendix
% \onecolumn
\clearpage
% \setcounter{page}{1}
% \maketitlesupplementary
\section{Proof of Theorems}\label{proof-theory}
\subsection{Proof of Theorem~\ref{thm:convergence}}
We first show what happens during one step of optimization, where we highlight the importance of analyzing delta error.

\begin{theorem}\label{thm-converge-gen}
    Consider model parameter $\theta\sim \pi$ and an aggregation rule \texttt{Aggr}$(\cdot)$ with step size  $\mu>0$. Define the updated parameter as
    \begin{align}
    \theta^+ := \theta - \mu \widehat g_{\texttt{Aggr}}(\theta).
    \end{align}
    Assuming the gradient $\nabla\mathcal{L}(\theta)$ is $\gamma$-Lipschitz in $\theta$ for any input, and let the step size $\mu \leq \tfrac{1}{\gamma},$ we have 
    \begin{align}
        \E_{\widehat\gD_a, \theta}[\mathcal{L}_{\gD_a}(\theta^+)-\mathcal{L}_{\gD_a}(\theta)]\leq -\tfrac{\mu}{2}(\|g_a\|_\pi^2-\Delta^2_{\texttt{Aggr}}).
    \end{align}
    
\end{theorem}
\begin{proof}
    The proof is the same as Theorem A.1 in~\citep{anonymous2024principled}.
\end{proof}

 \begin{theorem}[Convergence of \texttt{Aggr}$(\cdot)$]\label{thm:convergence}
     For any probability measure $\pi$ over the parameter space, and an aggregation rule \texttt{Aggr}$(\cdot)$ with step size  $\mu>0$. We update the parameter for $T$ steps by 
    $
    \theta^{t+1} := \theta^{t} - \mu \widehat g_{\texttt{Aggr}}(\theta^t). 
    $
    Assume the gradient $\nabla\mathcal{L}(\theta)$ and $\widehat g_{\texttt{Aggr}}(\theta)$ are $\tfrac{\gamma}{2}$-Lipschitz in $\theta$ such that $\theta^t\to \widehat\theta_{\texttt{Aggr}}$.
    %, and the gradients are bounded over the optimization path, i.e., $\|\nabla\ell(\theta^t, z)\|+\|\widehat g_{\texttt{Aggr}}(\theta^t)\|\leq B$. 
    $\Delta_{\texttt{Aggr\_max}}$ is the Delta error at time $t^{\prime}$ when $\|\widehat g_{\texttt{Aggr}}(\widehat\theta_{\texttt{Aggr}})-\nabla\mathcal{L}_{\gD_a^{t^{\prime}}}(\widehat\theta_{\texttt{Aggr}})\|^2$ is maximized. Then, given step size $\mu \leq \tfrac{1}{\gamma}$ and a small enough $\epsilon>0$,  with probability at least $1-\delta$ we have
    \begin{align}
        \|\nabla\mathcal{L}_{\gD_a^T}(\theta^T)\|^2 \leq \frac{1}{\delta^2}\left(\sqrt{C_\epsilon \cdot \Delta^2_{\texttt{Aggr\_max}} }+\gO(\epsilon) \right)^2+ \mathcal{O}\left(\frac{1}{T}\right)+\mathcal{O}(\epsilon),
    \end{align}
    where $C_\epsilon=\E_{\widehat \gD_a^{t^\prime}} [ {1}/{\pi(B_\epsilon(\widehat\theta_{\texttt{Aggr}}))}]^2 $ and $B_\epsilon(\widehat\theta_{\texttt{Aggr}})\subset \sR^m$ is the ball with radius $\epsilon$ centered at $\widehat\theta_{\texttt{Aggr}}$. The $C_\epsilon$ measures how well $\pi$ covers where the optimization goes.
 \end{theorem}

\begin{proof}
    Denote random function $\widehat f:\sR^m\to \sR_+$ as 
    \begin{align}
        \widehat f(\theta)=\|\widehat g_{\texttt{Aggr}}(\theta)-\nabla\mathcal{L}_{\gD_a}(\theta)\|, \label{eq:thm-gen-99}
    \end{align}
    where the randomness comes from $\widehat\gD_a$. Note that $\widehat f$ is $\gamma$-Lipschitz by assumption. Now we consider $B_\epsilon(\widehat\theta_{\texttt{Aggr}})\subset \sR^m$, i.e., the ball with radius $\epsilon$ centered at $\widehat\theta_{\texttt{Aggr}}$. Then, by $\gamma$-Lipschitzness we have
    \begin{align}
        \E_{\theta \sim \pi} \widehat f(\theta) = \int \widehat f(\theta) \diff \pi(\theta)\\
        \geq \int_{B_\epsilon(\widehat\theta_{\texttt{Aggr}})} (\widehat f(\widehat\theta_{\texttt{Aggr}})-\gamma \epsilon)  \diff \pi(\theta)\\
        =(\widehat f(\widehat\theta_{\texttt{Aggr}})-\gamma \epsilon) \pi(B_\epsilon(\widehat\theta_{\texttt{Aggr}}))
    \end{align}
    Therefore,
    \begin{align}
        \widehat f(\widehat\theta_{\texttt{Aggr}})\leq \frac{1}{\pi(B_\epsilon(\widehat\theta_{\texttt{Aggr}}))}\cdot \E_{\theta\sim \pi} \widehat f(\theta)+\gO(\epsilon).
    \end{align}
    Taking expectation w.r.t. $\widehat \gD_a$ on both sides, we have
    \begin{align}
        \E_{\widehat \gD_a}\widehat f(\widehat\theta_{\texttt{Aggr}})&\leq \E_{\widehat \gD_a} \left[ \frac{1}{\pi(B_\epsilon(\widehat\theta_{\texttt{Aggr}}))}\cdot \E_{\theta\sim \pi} \widehat f(\theta)\right] +\gO(\epsilon)\\
        &\leq \sqrt{\E_{\widehat \gD_a} \left[ \frac{1}{\pi(B_\epsilon(\widehat\theta_{\texttt{Aggr}}))}\right]^2 \cdot \E_{\widehat \gD_a} \left[\E_{\theta\sim \pi} \widehat f(\theta)\right]^2 }+\gO(\epsilon) \tag{Cauchy-Schwarz}\\
        &=\sqrt{C_\epsilon \cdot \E_{\widehat \gD_a} \left[\E_{\theta\sim \pi} \widehat f(\theta)\right]^2 }+\gO(\epsilon) \tag{by definition of $C_\epsilon$}\\
        &\leq \sqrt{C_\epsilon \cdot \E_{\widehat \gD_a}\E_{\theta\sim \pi}  \left[\widehat f(\theta)\right]^2 }+\gO(\epsilon)\tag{Jensen's inequality}\\
        &=\sqrt{C_\epsilon \cdot \Delta^2_{\texttt{Aggr}} }+\gO(\epsilon) 
    \end{align}
    By Markov's inequality, with probability at least $1-\delta$  we have a sampled dataset $\widehat \gD_a$ such that
    \begin{align}
        \widehat f(\widehat\theta_{\texttt{Aggr}}) \leq \frac{1}{\delta}\E_{\widehat \gD_a}\widehat f(\widehat\theta_{\texttt{Aggr}}) \leq \frac{1}{\delta}\sqrt{C_\epsilon \cdot \Delta^2_{\texttt{Aggr}} }+\gO(\epsilon/\delta) \label{eq:thm-gen-100}
    \end{align}
    Conditioned on such event, we proceed on to the optimization part.

    Note that Theorem~\ref{thm-converge-gen} characterizes how the optimization works for one gradient update. We denote $\gD_a^t$ as the data distribution $\gD_a$ at time step $t$. Therefore, for any time step $t=0,\dots,T-1$, we can apply Theorem~\ref{thm-converge-gen} which only requires the Lipschitz assumption:
    \begin{align}
        \mathcal{L}_{\gD_a^t}(&\theta^{t+1})-\mathcal{L}_{\gD_a^t}(\theta^t)\leq -\frac{\mu}{2} \left(\| \nabla\mathcal{L}_{\gD_a^t}(\theta^t)\|^2- \| \widehat g_{\texttt{Aggr}}(\theta^t)-\nabla\mathcal{L}_{\gD_a^t}(\theta^t)\|^2\right).
    \end{align}
    We notice that $\gD_a^t$ changes based on $\theta$s of different time steps. On both sides, to sum over $t=0,\dots,T-1$, we first consider two terms: 

    \begin{align}
        (\mathcal{L}_{\gD_a^t}(&\theta^{t+1})-\mathcal{L}_{\gD_a^t}(\theta^t)) + (\mathcal{L}_{\gD_a^{t-1}}(\theta^{t})-\mathcal{L}_{\gD_a^{t-1}}(\theta^{t-1}))
    \end{align}

    To compare $\mathcal{L}_{\gD_a^t}(\theta^t)$ and $\mathcal{L}_{\gD_a^{t-1}}(\theta^{t})$, since $\gD_a^t$ optimizes one more step than $\gD_a^{t-1}$, we assume $\mathcal{L}_{\gD_a^t}(\theta^t) \leq \mathcal{L}_{\gD_a^{t-1}}(\theta^{t})$ for $\forall t$. Therefore, we have:

    \begin{align}
        (\mathcal{L}_{\gD_a^t}(&\theta^{t+1})-\mathcal{L}_{\gD_a^t}(\theta^t)) + (\mathcal{L}_{\gD_a^{t-1}}(\theta^{t})-\mathcal{L}_{\gD_a^{t-1}}(\theta^{t-1})) \geq (\mathcal{L}_{\gD_a^t}(\theta^{t+1})-\mathcal{L}_{\gD_a^{t-1}}(\theta^t)) + (\mathcal{L}_{\gD_a^{t-1}}(\theta^{t})-\mathcal{L}_{\gD_a^{t-1}}(\theta^{t-1}))
    \end{align}

    Summing up all time steps,

    \begin{align}
        \mathcal{L}_{\gD_a^t}(&\theta^{t+1})-\mathcal{L}_{\gD_a^0}(\theta^0)\leq 
        \mathcal{L}_{\gD_a^{t}}(\theta^{t+1})- \mathcal{L}_{\gD_a^{t-1}}(\theta^{t}) + \mathcal{L}_{\gD_a^{t-1}}(\theta^{t}) -\mathcal{L}_{\gD_a^{t-2}}(\theta^{t-1}) + ... - \mathcal{L}_{\gD_a^0}(\theta^0) \\&\leq -\frac{\mu}{2} \left(\sum_{t=0}^{T-1}\| \nabla\mathcal{L}_{\gD_a^{t-1}}(\theta^t)\|^2- \sum_{t=0}^{T-1}\| \widehat g_{\texttt{Aggr}}(\theta^t)-\nabla\mathcal{L}_{\gD_a^t}(\theta^t)\|^2\right).
    \end{align}
    Dividing both sides by $T$, and with regular algebraic manipulation we derive
    \begin{align}
        \frac{1}{T}\sum_{t=0}^{T-1}\| \nabla\mathcal{L}_{\gD_a^{t-1}}(\theta^t)\|^2\leq \frac{2}{\mu T}(\mathcal{L}_{\gD_a^0}(&\theta^{0})-\mathcal{L}_{\gD_a^{T-1}}(\theta^T))+\frac{1}{T}\sum_{t=0}^{T-1}\| \widehat g_{\texttt{Aggr}}(\theta^t)-\nabla\mathcal{L}_{\gD_a^t}(\theta^t)\|^2.
    \end{align}
    Note that we assume the loss function $\mathcal{L}_\gD(\theta):=\E_{(x,y)\sim \gD}\mathcal{L}(f(x), y)$, is non-negative. Thus, we have
\begin{align}
        \frac{1}{T}\sum_{t=0}^{T-1}\| \nabla\mathcal{L}_{\gD_a^{t-1}}(\theta^t)\|^2&\leq \frac{2\mathcal{L}_{\gD_a^0}(\theta^{0})}{\mu T}+\frac{1}{T}\sum_{t=0}^{T-1}\| \widehat g_{\texttt{Aggr}}(\theta^t)-\nabla\mathcal{L}_{\gD_a^t}(\theta^t)\|^2.\label{eq:thm-gen-1}
    \end{align}

    Note that we assume given $\widehat\gD_a$ we have $\theta^t\to \widehat\theta_{\texttt{Aggr}}$. Therefore, for any $\epsilon>0$ there exist $T_\epsilon$ such that
    \begin{align}
        \forall t>T_\epsilon: \|\theta^t-\widehat\theta_{\texttt{Aggr}}\|<\epsilon. \label{eq:thm-gen-13}
    \end{align}
    This implies that $\forall t>T_\epsilon$:
    \begin{align}
        \mu\| \widehat g_{\texttt{Aggr}}(\theta^t)\|=\|\theta^{t+1}-\widehat\theta_{\texttt{Aggr}}+\widehat\theta_{\texttt{Aggr}}-\theta^t\|\leq \|\theta^{t+1}-\widehat\theta_{\texttt{Aggr}}\|+\|\widehat\theta_{\texttt{Aggr}}-\theta^t\|< 2\epsilon. \label{eq:thm-gen-12}
    \end{align}
    Moreover, \eqref{eq:thm-gen-13} also implies $\forall t_1, t_2>T_\epsilon$:
    \begin{align}
        \| \nabla\mathcal{L}_{\gD_a^{t_1}}(\theta^{t_1})-\nabla\mathcal{L}_{\gD_a^{t_2}}(\theta^{t_2})\|&\leq \gamma \|\theta^{t_1}-\theta^{t_2}\|\tag{$\gamma$-Lipschitzness}\\
        &<2\epsilon. \label{eq:thm-gen-14}
    \end{align}
    % The above inequality means that $\{\nabla\mathcal{L}_{\gD_a}(\theta^{t})\}_{t}$ is a Cauchy sequence. 

    Now, let's get back to \eqref{eq:thm-gen-1}. For $\forall T>T_\epsilon$ we have
    \begin{align}
        \frac{1}{T}\sum_{t=0}^{T-1}\| \nabla\mathcal{L}_{\gD_a^{t-1}}(\theta^t)\|^2&\leq \frac{2\mathcal{L}_{\gD_a^0}(\theta^{0})}{\mu T}+\frac{1}{T}\sum_{t=0}^{T_\epsilon-1}\| \widehat g_{\texttt{Aggr}}(\theta^t)-\nabla\mathcal{L}_{\gD_a^t}(\theta^t)\|^2+\frac{1}{T} \sum_{t=T_\epsilon}^{T-1}\| \widehat g_{\texttt{Aggr}}(\theta^t)-\nabla\mathcal{L}_{\gD_a^t}(\theta^t)\|^2\\
        &=\mathcal{O}\left(\frac{1}{T}\right)+\frac{1}{T} \sum_{t=T_\epsilon}^{T-1}\| \widehat g_{\texttt{Aggr}}(\theta^t)-\nabla\mathcal{L}_{\gD_a^t}(\theta^t)\|^2\\
        &=\mathcal{O}\left(\frac{1}{T}\right)+\frac{1}{T} \sum_{t=T_\epsilon}^{T-1}\| \widehat g_{\texttt{Aggr}}(\theta^t)- \widehat g_{\texttt{Aggr}}(\widehat\theta_{\texttt{Aggr}})+ \widehat g_{\texttt{Aggr}}(\widehat\theta_{\texttt{Aggr}})-\nabla\mathcal{L}_{\gD_a^t}(\theta^t)\|^2\\
        &\leq \mathcal{O}\left(\frac{1}{T}\right)+\frac{1}{T} \sum_{t=T_\epsilon}^{T-1} \left( \| \widehat g_{\texttt{Aggr}}(\theta^t)- \widehat g_{\texttt{Aggr}}(\widehat\theta_{\texttt{Aggr}})\|+ \|\widehat g_{\texttt{Aggr}}(\widehat\theta_{\texttt{Aggr}})-\nabla\mathcal{L}_{\gD_a^t}(\theta^t)\|\right)^2 \tag{triangle inequality}\\
        &=\mathcal{O}\left(\frac{1}{T}\right)+\frac{1}{T} \sum_{t=T_\epsilon}^{T-1} \left( \mathcal{O}(\epsilon)+\|\widehat g_{\texttt{Aggr}}(\widehat\theta_{\texttt{Aggr}})-\nabla\mathcal{L}_{\gD_a^t}(\theta^t)\|\right)^2 \tag{by \eqref{eq:thm-gen-12}}\\
        &=\mathcal{O}\left(\frac{1}{T}\right)+\mathcal{O}(\epsilon)+\frac{1}{T} \sum_{t=T_\epsilon}^{T-1} \left(\|\widehat g_{\texttt{Aggr}}(\widehat\theta_{\texttt{Aggr}})-\nabla\mathcal{L}_{\gD_a^t}(\widehat\theta_{\texttt{Aggr}})+\nabla\mathcal{L}_{\gD_a^t}(\widehat\theta_{\texttt{Aggr}})-\nabla\mathcal{L}_{\gD_a^t}(\theta^t)\|\right)^2 \\
        &\leq \mathcal{O}\left(\frac{1}{T}\right)+\mathcal{O}(\epsilon)+\frac{1}{T} \sum_{t=T_\epsilon}^{T-1} \left(\|\widehat g_{\texttt{Aggr}}(\widehat\theta_{\texttt{Aggr}})-\nabla\mathcal{L}_{\gD_a^t}(\widehat\theta_{\texttt{Aggr}})\|+\gO(\epsilon)\right)^2 \tag{by \eqref{eq:thm-gen-14}}\\
        &\leq \mathcal{O}\left(\frac{1}{T}\right)+\mathcal{O}(\epsilon)+\|\widehat g_{\texttt{Aggr}}(\widehat\theta_{\texttt{Aggr}})-\nabla\mathcal{L}_{\gD_a^{t^{\prime}}}(\widehat\theta_{\texttt{Aggr}})\|^2 \label{eq:thm-gen-101}
    \end{align}
    Equation~\ref{eq:thm-gen-101} bounds the left hand side with the maximum $\|\widehat g_{\texttt{Aggr}}(\widehat\theta_{\texttt{Aggr}})-\nabla\mathcal{L}_{\gD_a^t}(\widehat\theta_{\texttt{Aggr}})\|^2$ one can get during the optimization steps. Here, we assume at time $t^{\prime}$, the largest value is attained. We denote $\Delta^2_{\texttt{Aggr\_max}}$ as the delta error at time step $t'$.
    
    Then, we can continue with what we have done at the beginning of the proof of this theorem:
    \begin{align}
    \eqref{eq:thm-gen-101}&=\mathcal{O}\left(\frac{1}{T}\right)+\mathcal{O}(\epsilon)+f(\widehat\theta_{\texttt{Aggr}})^2 \tag{by \eqref{eq:thm-gen-99}}\\
    &\leq \mathcal{O}\left(\frac{1}{T}\right)+\mathcal{O}(\epsilon)+\left(\frac{1}{\delta}\sqrt{C_\epsilon \cdot \Delta^2_{\texttt{Aggr\_max}} }+\gO(\epsilon/\delta) \right)^2  \tag{by \eqref{eq:thm-gen-100}}
    \end{align}
    Therefore, combining the above we finally have: for $\forall T>T_\epsilon$ with probability at least $1-\delta$,
    \begin{align}
        \frac{1}{T}\sum_{t=0}^{T-1}\| \nabla\mathcal{L}_{\gD_a^{t-1}}(\theta^t)\|^2 \leq \mathcal{O}\left(\frac{1}{T}\right)+\mathcal{O}(\epsilon)+\frac{1}{\delta^2}\left(\sqrt{C_\epsilon \cdot \Delta^2_{\texttt{Aggr\_max}} }+\gO(\epsilon) \right)^2 \label{eq:thm-gen-103}
    \end{align}

    To complete the proof, let us investigate the left-hand side. 
    \begin{align}
        \frac{1}{T}\sum_{t=0}^{T-1}\| \nabla\mathcal{L}_{\gD_a^{t-1}}(\theta^t)\|^2&=\frac{1}{T} \sum_{t=0}^{T_\epsilon -1}\| \nabla\mathcal{L}_{\gD_a^t}(\theta^t)\|^2+\frac{1}{T} \sum_{t=T_\epsilon}^{T-1}\| \nabla\mathcal{L}_{\gD_a^t}(\theta^t)\|^2\\
        &=\gO\left(\frac{1}{T}\right)+\frac{1}{T} \sum_{t=T_\epsilon}^{T-1}\| \nabla\mathcal{L}_{\gD_a^t}(\theta^t)\|^2\\
        &\geq \gO\left(\frac{1}{T}\right)+\frac{1}{T} \sum_{t=T_\epsilon}^{T-1}\left(\| \nabla\mathcal{L}_{\gD_a^t}(\theta^t)-\nabla\mathcal{L}_{\gD_a^T}(\theta^T)\| - \|\nabla\mathcal{L}_{\gD_a^T}(\theta^T)\|\right)^2\tag{triangle inequality}\\
        &=\gO\left(\frac{1}{T}\right)+\frac{1}{T} \sum_{t=T_\epsilon}^{T-1}\left(\gO(\epsilon)+\|\nabla\mathcal{L}_{\gD_a^T}(\theta^T)\|^2\right) \tag{by \eqref{eq:thm-gen-14}}\\
        &=\gO\left(\frac{1}{T}\right)+\gO(\epsilon)+\|\nabla\mathcal{L}_{\gD_a^T}(\theta^T)\|^2. \label{eq:thm-gen-104}
    \end{align}
    Combining \eqref{eq:thm-gen-103} and \eqref{eq:thm-gen-104}, we finally have
    \begin{align}
        \|\nabla\mathcal{L}_{\gD_a^T}(\theta^T)\|^2 \leq \mathcal{O}\left(\frac{1}{T}\right)+\mathcal{O}(\epsilon)+\frac{1}{\delta^2}\left(\sqrt{C_\epsilon \cdot \Delta^2_{\texttt{Aggr\_max}} }+\gO(\epsilon) \right)^2,
    \end{align}
    which completes the proof.
\end{proof}

{\subsection{Proof of Theorem~\ref{thm:error-GP}}

To prove Theorem~\ref{thm:error-GP}, we first use the following definitions and lemmas from~\citep{anonymous2024principled}, to get the delta errors of Gradient Projection (GP) and standard adversarial training (AT):

\begin{definition}[GP Aggregation] \label{def:gp_aggr}
Let $\beta \in [0,1]$ be the weight that balances between $\widehat g_a$ and $\widehat g_n$. 
The GP aggregation operation is
\begin{align} 
    GP (\widehat g_a, \widehat g_n) = \left( (1-\beta)\widehat g_a + \beta \texttt{Proj}_+(\widehat g_a|\widehat g_n) \right).
\end{align}
where $\texttt{Proj}_+(\widehat g_a|\widehat g_n) =\max\{\langle \widehat g_a, \widehat g_n \rangle, 0 \}\widehat g_n/\|\widehat g_n\|^2$ is the operation that projects $\widehat g_a$ to the positive direction of $\widehat g_n$.
\end{definition}

\begin{definition}[AT Aggregation] \label{def:at_aggr} 
The AT aggregation operation is
\begin{align} 
    AT (\widehat g_a) = \widehat g_a.
\end{align}
standard AT only leverages the gradient update on $\widehat \gD_a$.
\end{definition}

\begin{lemma}[Delta Error of GP]\label{lemma:GP}
Given distributions $\widehat\gD_a$, $\gD_a$ and $\widehat\gD_{n}$, as well as the model updates $\widehat{g_a}, g_a, \widehat g_n$ on these distributions per epoch, we have $\Delta^2_{{GP}}$ as follows
\begin{align} 
    \Delta^2_{{GP}}\approx \left((1-\beta)^2+\frac{2\beta-\beta^2}{m}\right) \E_{\widehat\gD_a} \|g_{a} -  \widehat g_{a}\|^2_\pi   + \beta^2 \bar\tau^2 \|g_{a} -  
    \widehat g_{n}\|^2_\pi,
\end{align}

In the above equation, $m$ is the model dimension and $\bar \tau^2= \E_\pi[\tau^2] \in [0, 1]$ where $\tau (\theta)$ is the $\sin(\cdot)$ value of the angle between $\widehat g_{n}$ and $g_{a}-\widehat g_{n}$. $\|\cdot\|_{\pi}$ is the $\pi$-norm over the model parameter space.

\end{lemma}

\begin{proof}
    The proof is the same as Theorem 4.4 in~\citet{anonymous2024principled}.
\end{proof}

\begin{lemma}[Delta Error of AT]\label{lemma:AT}
Given distributions $\widehat\gD_a$, $\gD_a$ and $\widehat\gD_{n}$, as well as the model updates $\widehat{g_a}, g_a, \widehat g_n$ on these distributions per epoch, we have $\Delta^2_{AT}$ as follows
\begin{align} 
    \Delta^2_{AT} = \E_{\widehat\gD_a} \|g_{a} -  \widehat g_{a}\|^2_\pi,
\end{align}

where $\|\cdot\|_{\pi}$ is the $\pi$-norm over the model parameter space.
\end{lemma}

Then, we prove Theorem~\ref{thm:error-GP}.
\label{thm:error-diff}\begin{theorem}[Error Analysis of GP]
When the model dimension is large ($m \rightarrow \infty$) at time step $t$, we have
\begin{align} 
    \Delta^2_{AT} - \Delta^2_{GP} \approx \beta (2 - \beta) \E_{\widehat\gD_a^t} \|g_{a} - \widehat{g_{a}} \|^2_\pi - \beta^2 \bar\tau^2 \|g_{a} -  \widehat g_{n}\|^2_\pi. \label{eq:thm-fedgp}
\end{align}
% is the Delta error when only considering $\gD_{S_i}$ as the source domain. $m$ is the model dimension and In the above equation, 
$\bar \tau^2= \E_\pi[\tau^2] \in [0, 1]$ where $\tau$ is the $\sin(\cdot)$ value of the angle between $\widehat g_{n}$ and $g_{a}-\widehat g_{n}$, $\|\cdot\|_{\pi}$ is the $\pi$-norm over the model parameter space. 
% %Moreover, the inequality~\eqref{eq:thm-fedgp} becomes equality when $N=1$.
\end{theorem}
    
\begin{proof}
 $\Delta^2_{\texttt{AT}} - \Delta^2_{\texttt{GP}}\\ \approx
  \E_{\widehat\gD_a^t} \|g_{a} - \widehat{g_{a}} \|^2_\pi - \left((1-\beta)^2+\frac{2\beta-\beta^2}{m}\right) \E_{\widehat\gD_a} \|g_{a} -  \widehat g_{a}\|^2_\pi   - \beta^2 \bar\tau^2 \|g_{a} -  \widehat g_{n}\|^2_\pi \\
    = \left(1 - ((1-\beta)^2+\frac{2\beta-\beta^2}{m}) \right) \E_{\widehat\gD_a^t} \|g_{a} - \widehat{g_{a}} \|^2_\pi - \beta^2 \bar\tau^2 \|g_{a} -  \widehat g_{n}\|^2_\pi \\
= (1 + \frac{1}{m}) \beta (2 - \beta) \E_{\widehat\gD_a^t} \|g_{a} - \widehat{g_{a}} \|^2_\pi - \beta^2 \bar\tau^2 \|g_{a} -  \widehat g_{n}\|^2_\pi$
% \begin{align}
    % \E_{\widehat\gD_a^t} \|g_{a} - \widehat{g_{a}} \|^2_\pi - \left((1-\beta)^2+\frac{2\beta-\beta^2}{m}\right) \|g_{a} -  \widehat g_{n}\|^2_\pi  - \beta^2 \bar\tau^2 \E_{\widehat\gD_a^t} \|g_{a} -  \widehat g_{a}\|^2_\pi    \\
    % = (1 - \beta^2 \bar\tau^2) \E_{\widehat\gD_a^t} \|g_{a} -  \widehat g_{a}\|^2_\pi -  \left((1-\beta)^2+\frac{2\beta-\beta^2}{m}\right) \|g_{a} -  \widehat g_{n}\|^2_\pi \\
    % = (1 - \beta^2 \bar\tau^2) \E_{\widehat\gD_a^t} \|g_{a} -  \widehat g_{a}\|^2_\pi -  (1-\beta)^2 \|g_{a} - \widehat g_{n}\|^2_\pi -  \frac{2\beta-\beta^2}{m}  \|g_{a} - \widehat g_{n}\|^2_\pi\\
    % = \left( (1 - \beta^2 \bar\tau^2) \E_{\widehat\gD_a^t} \|g_{a} -  \widehat g_{a}\|^2_\pi - (1-\beta^2) \|g_{a} - \widehat g_{n}\|^2_\pi\right)\\ + 2\beta\|g_{a} - \widehat g_{n}\|^2_\pi -  \frac{2\beta-\beta^2}{m}  \|g_{a} - \widehat g_{n}\|^2_\pi$
% \end{align}

When $m \rightarrow \infty$, we have a simplified version of the error difference as follows
\begin{align}
    \Delta^2_{AT} - \Delta^2_{GP} \approx \beta (2 - \beta) \E_{\widehat\gD_a^t} \|g_{a} - \widehat{g_{a}} \|^2_\pi - \beta^2 \bar\tau^2 \|g_{a} -  \widehat g_{n}\|^2_\pi
\end{align}

% \begin{align*}
%     \Delta^2_{AT} - \Delta^2_{GP} \approx  (1 - \beta^2 \bar\tau^2) \E_{\widehat\gD_a^t} \|g_{a} -  \widehat g_{a}\|^2_\pi + (2\beta - (1-\beta^2)) \|g_{a} - \widehat g_{n}\|^2_\pi
% \end{align*}

\end{proof}

\noindent\textbf{Interpretation.} When $\beta = 0.5$, we can usually show $\Delta^2_{AT} > \Delta^2_{GP}$, because $\beta (2 - \beta) > \beta^2 \bar\tau^2 (0.75 > 0.25)$ for the coefficients of two terms. We estimate the actual values of terms $E_{\widehat{D}_{a^t}} \|g_{a} - \widehat{g_{a}} \|^2_\pi$ (variance), $\|g_{a} -  \widehat g_{n}\|^2_\pi$ (bias), and $\bar\tau$ using the estimation methods in~\citet{anonymous2024principled}. Table~\ref{table:error-analysis} displays the values of those terms as well as the error differences on ResNet18 experiments at epoch $5, 10, 15, 20, 60$. We plot the changing of these terms on the ResNet18 experiment in Figure~\ref{fig:error-diff}. The order of difference is always positive and usually smaller than $1e^{-08}$ and approaches the order of $1e^{-12}$ in the end. 

\begin{table}[!h]
\centering
\caption{Estimations the actual values of terms $E_{\widehat{D_{a^t}}} \|g_{a} - \widehat{g_{a}} \|^2_\pi$ (variance), $\|g_{a} -  \widehat g_{n}\|^2_\pi$ (bias), $\bar\tau$, and $\Delta^2_{AT} - \Delta^2_{GP}$ (error differences) across different epochs.}
\label{table:error-analysis}
\begin{tabular}{lrrrrr}
Terms / epochs & 5          & 10         & 15         & 20         & 60         \\\hline
$E_{\widehat{D}_{a^t}} \|g_{a} - \widehat{g_{a}} \|^2_\pi$                   & 4.6017e-08 & 2.0448e-09 & 6.9623e-10 & 6.4329e-10 & 2.3849e-11 \\
$\|g_{a} -  \widehat g_{n}\|^2_\pi$                        & 0.0007     & 9.9098e-05 & 4.4932e-05 & 3.7930e-05 & 2.8391e-06 \\
$\bar\tau$                       & 0.0071     & 0.0052     & 0.0036     & 0.0038     & 0.0030     \\
$\Delta^2_{AT} - \Delta^2_{GP}$           & 2.5335e-08 & 8.5709e-10 & 3.7609e-10 & 3.4487e-10 & 1.1574e-11\\\hline
\end{tabular}
\end{table}

\begin{figure}
    \centering
    \includegraphics[width=0.75\linewidth]{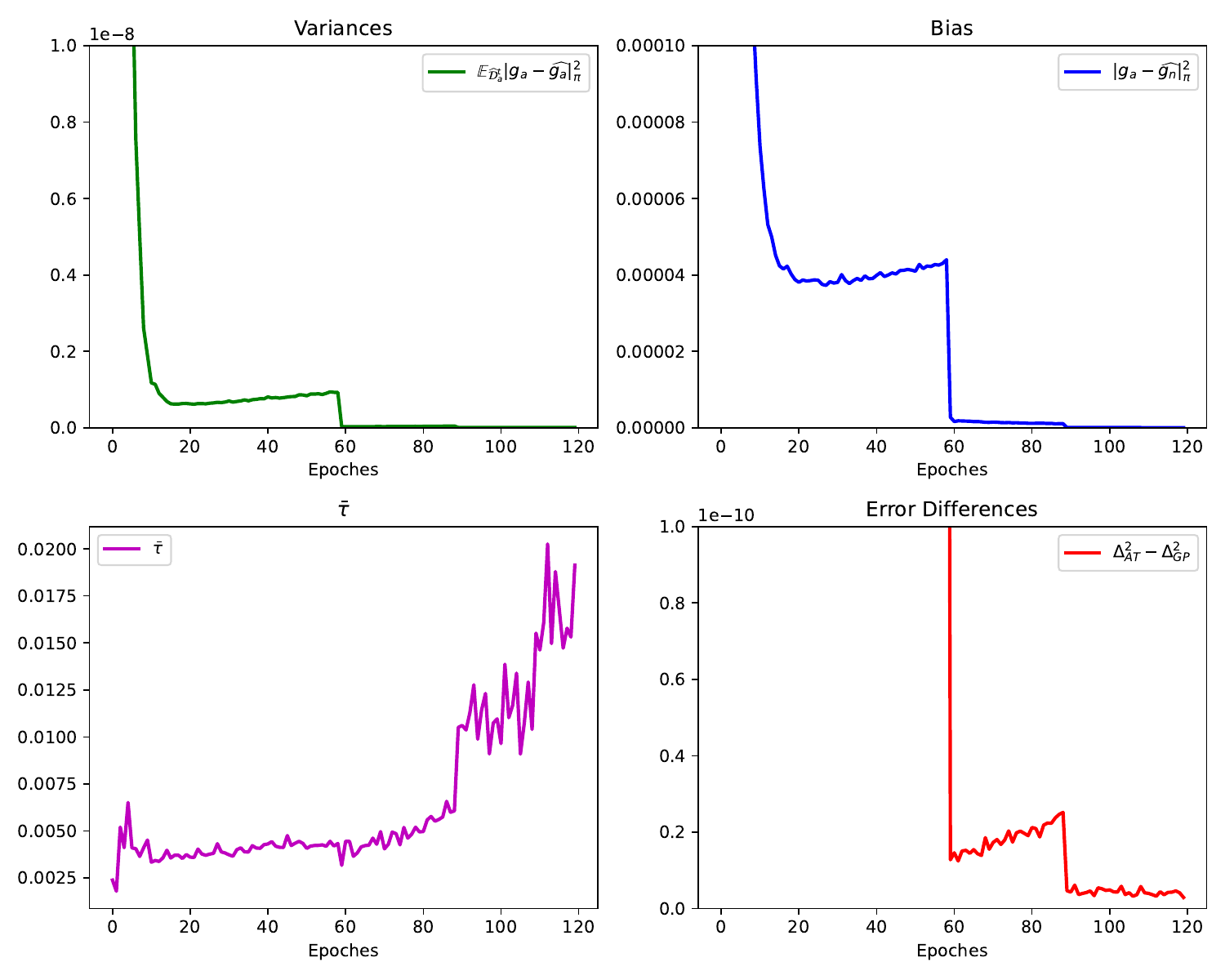}
    \caption{Plot of values of terms $E_{\widehat{D}_{a^t}} \|g_{a} - \widehat{g_{a}} \|^2_\pi$ (variance), $\|g_{a} -  \widehat g_{n}\|^2_\pi$ (bias), $\bar\tau$, and $\Delta^2_{AT} - \Delta^2_{GP}$ (error differences).}
    \label{fig:error-diff}
\end{figure}

% $(1-\beta^2) < 2\beta$ $(0.75 < 1)$ for the term $\|g_{a} -  \widehat g_{n}\|^2_\pi$. In the end, we get the difference as $0.75 \E_{\widehat\gD_a^t} \|g_{a} -  \widehat g_{a}\|^2_\pi + 0.25 \|g_{a} -  \widehat g_{n}\|^2_\pi$, which is always bigger than $0$.

\section{Additional Experiment Information}\label{supp}
In this section, we provide more training details, additional experiment results on the universal robustness of \textbf{RAMP} to common corruptions and unseen adversaries, runtime analysis of RAMP, additional ablation studies on different logit pairing losses, and AT from random initialization results on CIFAR-10 using WideResNet-28-10.
\subsection{More Training Details}
We set the batch size to $128$ for the experiments on ResNet-18 and WideResNet-28-10 architectures. We use an SGD optimizer with $0.9$ momentum and $5e^{-4}$ weight decay. For other experiments on ImageNet, we use a batch size of $64$ to fit into the GPU memory for larger models. For all training procedures, we select the last checkpoint for the comparison. When the pre-trained model was originally trained with extra data beyond the CIFAR-10 dataset, similar to~\citet{croce2022eat}, we use the extra $500$k images introduced by~\citet{carmon2019unlabeled} for fine-tuning, and each batch contains the same amount of standard and extra images. An epoch is completed when the whole standard training set has been used. 

% Doing logit pairing on $l_1$-AT ResNet-18 and $l_2$-AT ResNet-18 models result in a lower union accuracy than $l_\infty$-AT model with robust fine-tuning (Table~\ref{table:reverse-logit-pairing}).
\subsection{Runtime Analysis of RAMP}\label{time-analysis}
We present runtime analysis results demonstrating the fact that RAMP is more expensive than E-AT and less expensive than MAX in Table~\ref{table:runtime}. These results, recorded in seconds per epoch, were obtained using a single A100 40GB GPU. RAMP consistently supports that fact in all experiments. 

\begin{table}[!h]
\centering
\caption{Analysis of time per epoch for RAMP and related baselines. RAMP is more expensive than E-AT and less expensive than MAX.}
\begin{tabular}{lrrr}
Models \textbackslash Methods & E-AT~\citep{croce2022eat}                         & MAX    & \textbf{RAMP} \\\hline
CIFAR-10 RN-18 scratch           & 78                           & 219    & 157           \\
CIFAR-10 WRN-28-10 scratch       & 334                          & 1048   & 660           \\
CIFAR-10 RN-50                   & 188                          & 510    & 388           \\
CIFAR-10 WRN-34-20               & 1094                         & 2986   & 2264          \\
CIFAR-10 WRN-28-10 carmon        & {\color[HTML]{0E0E0E} 546}   & 1420   & 1110          \\
CIFAR-10 WRN-28-10 gowal         & {\color[HTML]{0E0E0E} 698}   & 1895   & 1456          \\
CIFAR-10 WRN-70-16               & {\color[HTML]{0E0E0E} 3486}  & 10330  & 7258          \\
ImageNet ResNet50             & {\color[HTML]{0E0E0E} 15656} & 41689  & 35038         \\
ImageNet Transformer          & 38003                        & 101646 & 81279        \\\hline
\end{tabular}\label{table:runtime}
\end{table}

\subsection{Additional Results on \textbf{RAMP} Generalizing to Common Corruptions and Unseen Adversaries for Universal Robustness}\label{genality-other}

In this section, we show \textbf{RAMP} can generalize better to other corruptions and unseen adversaries on union accuracy for stronger universal robustness. 

\textbf{Implementations.} For the $l_0$ attack, we use~\citet{croce2022sparse} with an epsilon of $9$ pixels and $5k$ query points. For common corruptions, we directly use the implementation of RobustBench~\citep{croce2020robustbench} for evaluation across 5 severity levels on all corruption types used in~\citet{hendrycks2019benchmarking}. For other unseen adversaries, we follow the implementation of~\citet{laidlaw2020perceptual}, where we set $eps=12$ for the fog attack, $eps = 0.5$ for the snow attack, $eps=60$ for the gabor attack, $eps=0.125$ for the elastic attack, and $eps=0.125$ for the jpeglinf attack with $100$ iterations. For ResNet-18 experiments, we do not compare with Winninghand~\citep{diffenderfer2021winning} since it uses a Wide-ResNet architecture. Also, we select the strongest baselines (E-AT and MAX) from the Wide-ResNet experiment results to compare for ResNet-18 experiments on universal robustness. 

\textbf{Results.} For the ResNet-18 training from scratch experiment on CIFAR-10, in Table~\ref{table:comm-corr-18} and ~\ref{table:unseen-18}, we also show \textbf{RAMP} generally outperforms by $0.5\%$ on common corruptions and $7\%$ on union accuracy against unseen adversaries compared with E-AT.

\begin{table}[!h]
\centering
\caption{Accuracy against common corruptions using ResNet-18 on CIFAR-10 dataset.}
\label{table:comm-corr-18}
\begin{tabular}{lr}
Models        & \multicolumn{1}{l}{common corruptions} \\\hline
% $l_1$-AT \\
% $l_2$-AT \\
% $l_\infty$-AT \\
E-AT          & 73.8                                   \\
MAX           & 75.1                                   \\
\textbf{RAMP} & 74.3       \\\hline                           
\end{tabular}
\end{table}

\begin{table}[!h]
\centering
\caption{Individual, average, and union accuracy against unseen adversaries using ResNet-18 on CIFAR-10 dataset.}
\label{table:unseen-18}
\begin{tabular}{lrrrrrrrr}
Models        & \multicolumn{1}{l}{$l_0$} & \multicolumn{1}{l}{fog} & \multicolumn{1}{l}{snow} & \multicolumn{1}{l}{gabor} & \multicolumn{1}{l}{elastic} & \multicolumn{1}{l}{jpeglinf} & \multicolumn{1}{l}{Avg} & \multicolumn{1}{l}{Union} \\\hline
% $l_1$-AT \\
% $l_2$-AT \\
% $l_\infty$-AT \\
E-AT          & 58.5                                    & 41.8                              & 30.8                               & 45.9                               & 55.0                                      & 59.1                                & 48.5                    & 18.8                      \\
MAX           & 70.8                                    & 40.0                                & 34.4                               & 45.1                               & 54.8                                    & 56.8                                & 50.3                    & 20.6                      \\
\textbf{RAMP} & 56.8                                    & 40.5                              & 40.5                               & 50.0                                 & 59.2                                    & 56.2                                & \textbf{50.5}           & \textbf{25.9}  \\\hline          
\end{tabular}
\end{table}

\subsection{Additional Experiments with Different Epsilon Values}\label{additional-diff-eps}

In this section, we provide additional results with different $\epsilon_1, \epsilon_2, \epsilon_\infty$ values. We select $\epsilon_\infty = [\frac{2}{255}, \frac{4}{255}, \frac{12}{255}, \frac{16}{255}]$, $\epsilon_1 = [6, 9, 12, 15]$, and $\epsilon_2 = [0.25, 0.75, 1.0, 1.5]$. We provide additional \textbf{RAMP} results compared with related baselines with training from scratch and performing robust fine-tuning in Section~\ref{vary-eps-scratch-res} and Section~\ref{vary-eps-finetune-res}, respectively. We observe that \textbf{RAMP} can surpass E-AT with significant margins as well as a better accuracy-robustness tradeoff for both training from scratch and robust fine-tuning with $\lambda = 2.0$ for training from scratch and $\lambda = 0.5$ for robust fine-tuning in most cases.

\subsubsection{Additional Results with Training from Scratch}\label{vary-eps-scratch-res}
\textbf{Changing $l_\infty$ perturbations with $\epsilon_\infty = [\frac{2}{255}, \frac{4}{255}, \frac{12}{255}, \frac{16}{255}]$.} Table~\ref{table:scratch-linf1} and Table~\ref{table:scratch-linf2} show that \textbf{RAMP} consistently outperforms E-AT~\citep{croce2022eat} on union accuracy when training from scratch.

\begin{table}[h!]
\caption{$(\epsilon_\infty = \mathbf{\frac{2}{255}}, \epsilon_1 = 12, \epsilon_2 = 0.5)$ and $(\epsilon_\infty = \mathbf{\frac{4}{255}}, \epsilon_1 = 12, \epsilon_2 = 0.5)$ with random initializations.}
    \begin{minipage}{.5\linewidth}
\centering
\vspace{2mm}
\begin{tabular}{lrrrrr}
     & Clean & $l_\infty$ & $l_2$   & $l_1$   & Union         \\\hline
E-AT & 87.2  & 73.3 & 64.1 & 55.4 & 55.4          \\
% MAX  & 85.6  & 72.1 & 63.6 & 56.4 & 56.4          \\
\bf RAMP & 86.3  & 73.3 & 64.9 & 59.1 & \textbf{59.1} \\\hline
\end{tabular}
\end{minipage}%
    \begin{minipage}{.5\linewidth}
\centering
\vspace{2mm}
\begin{tabular}{lrrrrr}
     & Clean & $l_\infty$ & $l_2$   & $l_1$   & Union         \\\hline
E-AT & 86.8  & 58.9 & 66.4 & 54.6 & 53.7          \\
% MAX  & 86.1  & 57.7 & 65.3 & 55.6 & 54.6          \\
\bf RAMP & 86.1  & 60.0 & 67.4 & 58.5 & \textbf{57.4}\\\hline
\end{tabular}
\end{minipage}
\label{table:scratch-linf1}
\end{table}

\begin{table}[h!]
\caption{$(\epsilon_\infty = \mathbf{\frac{12}{255}}, \epsilon_1 = 12, \epsilon_2 = 0.5)$ and $(\epsilon_\infty = \mathbf{\frac{16}{255}}, \epsilon_1 = 12, \epsilon_2 = 0.5)$ with random initializations.}
    \begin{minipage}{.5\linewidth}
\centering
\vspace{2mm}
\begin{tabular}{lrrrrr}
     & Clean & $l_\infty$ & $l_2$   & $l_1$   & Union         \\\hline
E-AT & 77.5  & 28.8 & 64.0 & 50.1 & 28.7          \\
% MAX  & 73.7  & 33.0 & 59.2 & 39.0 & 31.6          \\
\bf RAMP & 73.7  & 34.6 & 59.1 & 38.9 & \textbf{33.3} \\\hline
\end{tabular}
\end{minipage}%
    \begin{minipage}{.5\linewidth}
\centering
\vspace{2mm}
\begin{tabular}{lrrrrr}
     & Clean & $l_\infty$ & $l_2$   & $l_1$   & Union         \\\hline
E-AT & 69.4  & 18.8 & 58.7 & 47.7 & 18.7          \\
% MAX  & 65.3  & 24.1 & 51.3 & 32.0 & 23.2          \\
\bf RAMP & 65.0  & 25.7 & 49.8 & 32.6 & \textbf{25.0}\\\hline
\end{tabular}
\end{minipage}
\label{table:scratch-linf2}
\end{table}

\noindent\textbf{Changing $l_1$ perturbations with $\epsilon_1 = [6, 9, 12, 15]$.} Table~\ref{table:scratch-l11} and Table~\ref{table:scratch-l12} show that \textbf{RAMP} consistently outperforms E-AT~\citep{croce2022eat} on union accuracy when training from scratch.

\begin{table}[h!]
\caption{$(\epsilon_\infty = \frac{8}{255}, \mathbf{\epsilon_1 = 6}, \epsilon_2 = 0.5)$ and $(\epsilon_\infty = \frac{8}{255}, \mathbf{\epsilon_1 = 9}, \epsilon_2 = 0.5)$ with random initializations.}
    \begin{minipage}{.5\linewidth}
\centering
\vspace{2mm}
\begin{tabular}{lrrrrr}
     & Clean & $l_\infty$ & $l_2$   & $l_1$   & Union         \\\hline
E-AT & 85.5  & 43.1 & 67.9 & 63.9 & 42.8          \\
% MAX  & 83.1  & 47.2 & 63.9 & 53.6 & \textbf{45.6} \\
\bf RAMP & 83.8  & 48.1 & 63.0 & 51.2 & \bf46.0\\\hline
\end{tabular}
\end{minipage}%
    \begin{minipage}{.5\linewidth}
\centering
\vspace{2mm}
\begin{tabular}{lrrrrr}
     & Clean & $l_\infty$ & $l_2$   & $l_1$   & Union         \\\hline
E-AT & 84.6  & 41.8 & 67.7 & 57.6 & 41.4          \\
% MAX  & 82.9  & 45.6 & 65.6 & 50.8 & 44.2 \\
\bf RAMP & 82.6  & 47.5 & 65.7 & 50.8 & \textbf{45.9}\\\hline
\end{tabular}
\end{minipage}
\label{table:scratch-l11}
\end{table}

\begin{table}[h!]
\caption{$(\epsilon_\infty = \frac{8}{255}, \mathbf{\epsilon_1 = 15}, \epsilon_2 = 0.5)$ and $(\epsilon_\infty = \frac{8}{255}, \mathbf{\epsilon_1 = 18}, \epsilon_2 = 0.5)$ with random initializations.}
    \begin{minipage}{.5\linewidth}
\centering
\vspace{2mm}
\begin{tabular}{lrrrrr}
     & Clean & $l_\infty$ & $l_2$   & $l_1$   & Union         \\\hline
E-AT & 81.9  & 40.2 & 66.9 & 48.7 & 39.2          \\
% MAX  & 80.8  & 43.2 & 66.3 & 45.3 & 41.2          \\
\bf RAMP & 80.9  & 45.0 & 66.4 & 46.7 & \textbf{43.3} \\\hline
\end{tabular}
\end{minipage}%
    \begin{minipage}{.5\linewidth}
\centering
\vspace{2mm}
\begin{tabular}{lrrrrr}
     & Clean & $l_\infty$ & $l_2$   & $l_1$   & Union         \\\hline
E-AT & 81.0  & 39.8 & 65.8 & 44.3 & 38.0          \\
% MAX  & 79.9  & 42.2 & 65.7 & 42.8 & 39.9          \\
\bf RAMP & 79.9  & 43.5 & 65.7 & 45.0 & \textbf{41.9}\\\hline
\end{tabular}
\end{minipage}
\label{table:scratch-l12}
\end{table}

\noindent\textbf{Changing $l_2$ perturbations with $\epsilon_2 = [0.25, 0.75, 1.0, 1.5]$.} Table~\ref{table:scratch-l21} and Table~\ref{table:scratch-l22} show that \textbf{RAMP} consistently outperforms E-AT~\citep{croce2022eat} on union accuracy when training from scratch.

\begin{table}[h!]
\caption{$(\epsilon_\infty = \frac{8}{255}, \epsilon_1 = 12, \mathbf{\epsilon_2 = 0.25})$ and $(\epsilon_\infty = \frac{8}{255}, \epsilon_1 = 12, \mathbf{\epsilon_2 = 0.75})$ with random initializations.}
    \begin{minipage}{.5\linewidth}
\centering
\vspace{2mm}
\begin{tabular}{lrrrrr}
     & Clean & $l_\infty$ & $l_2$   & $l_1$   & Union         \\\hline
E-AT & 82.8  & 41.3 & 75.6 & 52.9 & 40.5          \\
% MAX  & 82.0  & 44.5 & 74.7 & 48.2 & 42.8          \\
\bf RAMP & 81.8  & 46.0 & 74.7 & 48.8 & \textbf{44.5} \\\hline
\end{tabular}
\end{minipage}%
    \begin{minipage}{.5\linewidth}
\centering
\vspace{2mm}
\begin{tabular}{lrrrrr}
     & Clean & $l_\infty$ & $l_2$   & $l_1$   & Union         \\\hline
E-AT & 83.0  & 41.2 & 57.6 & 53.0 & 40.5          \\
% MAX  & 81.9  & 44.3 & 48.2 & 48.5 & 43.0          \\
\bf RAMP & 81.9  & 46.1 & 56.9 & 48.7 & \textbf{44.5}\\\hline
\end{tabular}
\end{minipage}
\label{table:scratch-l21}
\end{table}

\begin{table}[h!]
\caption{$(\epsilon_\infty = \frac{8}{255}, \epsilon_1 = 12, \mathbf{\epsilon_2 = 1.0})$ and $(\epsilon_\infty = \frac{8}{255}, \epsilon_1 = 12, \mathbf{\epsilon_2 = 1.5})$ with random initializations.}
    \begin{minipage}{.5\linewidth}
\centering
\vspace{2mm}
\begin{tabular}{lrrrrr}
     & Clean & $l_\infty$ & $l_2$   & $l_1$   & Union         \\\hline
E-AT & 83.4  & 41.0 & 47.3 & 52.8 & 40.3          \\
% MAX  & 81.9  & 44.3 & 48.2 & 48.5 & 42.6          \\
\bf RAMP\\ ($\lambda$=5) & 81.5  & 46.0 & 46.5 & 48.1 & \textbf{44.1}\\\hline
\end{tabular}
\end{minipage}%
    \begin{minipage}{.5\linewidth}
\centering
\vspace{2mm}
\begin{tabular}{lrrrrr}
     & Clean & $l_\infty$ & $l_2$   & $l_1$   & Union         \\\hline
E-AT & 83.5  & 41.0 & 25.5 & 52.9 & 25.5          \\
% MAX  & 74.6  & 42.9 & 35.7 & 50.3 & 35.6          \\
\bf RAMP & 74.4  & 43.4 & 37.2 & 51.1 & \textbf{37.1}\\\hline
\end{tabular}
\end{minipage}
\label{table:scratch-l22}
\end{table}

\subsubsection{Additional Results with Robust Fine-tuning}\label{vary-eps-finetune-res}
\textbf{Changing $l_\infty$ perturbations with $\epsilon_\infty = [\frac{2}{255}, \frac{4}{255}, \frac{12}{255}, \frac{16}{255}]$.} Table~\ref{table:finetune-linf1} and Table~\ref{table:finetune-linf2} show that \textbf{RAMP} consistently outperforms E-AT~\citep{croce2022eat} on union accuracy when performing robust fine-tuning.

\begin{table}[h!]
\caption{$(\epsilon_\infty = \mathbf{\frac{2}{255}}, \epsilon_1 = 12, \epsilon_2 = 0.5)$ and $(\epsilon_\infty = \mathbf{\frac{4}{255}}, \epsilon_1 = 12, \epsilon_2 = 0.5)$ with robust fine-tuning.}
    \begin{minipage}{.5\linewidth}
\centering
\vspace{2mm}
\begin{tabular}{lrrrrr}
     & Clean & $l_\infty$ & $l_2$   & $l_1$   & Union         \\\hline
E-AT & 86.5  & 74.8 & 66.7 & 57.9 & 57.9          \\
% MAX  & 85.7  & 74.0 & 66.2 & 60.0 & 60.0          \\
\bf RAMP & 85.8  & 74.0 & 66.2 & 60.1 & \textbf{60.1} \\\hline
\end{tabular}
\end{minipage}%
    \begin{minipage}{.5\linewidth}
\centering
\vspace{2mm}
\begin{tabular}{lrrrrr}
     & Clean & $l_\infty$ & $l_2$   & $l_1$   & Union         \\\hline
E-AT & 85.9  & 61.4 & 67.9 & 57.6 & 56.8          \\
% MAX  & 85.8  & 61.2 & 67.8 & 59.4 & \textbf{58.4} \\
\bf RAMP & 85.7  & 60.9 & 67.6 & 59.3 & \textbf{58.1}\\\hline
\end{tabular}
\end{minipage}
\label{table:finetune-linf1}
\end{table}

\begin{table}[h!]
\caption{$(\epsilon_\infty = \mathbf{\frac{12}{255}}, \epsilon_1 = 12, \epsilon_2 = 0.5)$ and $(\epsilon_\infty = \mathbf{\frac{16}{255}}, \epsilon_1 = 12, \epsilon_2 = 0.5)$ with robust fine-tuning.}
    \begin{minipage}{.5\linewidth}
\centering
\vspace{2mm}
\begin{tabular}{lrrrrr}
     & Clean & $l_\infty$ & $l_2$   & $l_1$   & Union         \\\hline
E-AT & 75.5  & 30.8 & 62.4 & 44.6 & 30.0          \\
% MAX  & 74.1  & 33.2 & 60.1 & 38.1 & 31.4 \\
\bf RAMP & 74.0  & 33.6 & 59.7 & 38.5 & \textbf{31.9} \\\hline
\end{tabular}
\end{minipage}%
    \begin{minipage}{.5\linewidth}
\centering
\vspace{2mm}
\begin{tabular}{lrrrrr}
     & Clean & $l_\infty$ & $l_2$   & $l_1$   & Union         \\\hline
E-AT & 68.7  & 20.7 & 56.1 & 42.1 & 20.5          \\
% MAX  & 65.9  & 24.5 & 51.9 & 31.3 & 23.2 \\
\bf RAMP & 65.6  & 25.0 & 51.5 & 31.2 & \textbf{23.8}\\\hline
\end{tabular}
\end{minipage}
\label{table:finetune-linf2}
\end{table}

\noindent\textbf{Changing $l_1$ perturbations with $\epsilon_1 = [6, 9, 12, 15]$.} Table~\ref{table:scratch-l11} and Table~\ref{table:scratch-l12} show that \textbf{RAMP} consistently outperforms E-AT~\citep{croce2022eat} on union accuracy when performing robust fine-tuning.

\begin{table}[h!]
\caption{$(\epsilon_\infty = \frac{8}{255}, \mathbf{\epsilon_1 = 6}, \epsilon_2 = 0.5)$ and $(\epsilon_\infty = \frac{8}{255}, \mathbf{\epsilon_1 = 9}, \epsilon_2 = 0.5)$ with robust fine-tuning.}
    \begin{minipage}{.5\linewidth}
\centering
\vspace{2mm}
\begin{tabular}{lrrrrr}
     & Clean & $l_\infty$ & $l_2$   & $l_1$   & Union         \\\hline
E-AT & 84.2  & 45.8 & 66.8 & 59.0 & 45.0          \\
% MAX  & 83.3  & 48.3 & 64.2 & 53.1 & 46.3          \\
\bf RAMP \\
($\lambda$=1.5) & 83.0  & 48.7 & 63.5 & 51.7 & \textbf{46.4}\\\hline
\end{tabular}
\end{minipage}%
    \begin{minipage}{.5\linewidth}
\centering
\vspace{2mm}
\begin{tabular}{lrrrrr}
     & Clean & $l_\infty$ & $l_2$   & $l_1$   & Union         \\\hline
E-AT & 83.1  & 44.9 & 67.2 & 52.6 & 43.2          \\
% MAX  & 82.9  & 46.7 & 66.3 & 49.7 & 44.3          \\
\bf RAMP & 82.5  & 47.1 & 66.0 & 49.9 & \textbf{44.8}\\\hline
\end{tabular}
\end{minipage}
\label{table:finetune-l11}
\end{table}

\begin{table}[h!]
\caption{$(\epsilon_\infty = \frac{8}{255}, \mathbf{\epsilon_1 = 15}, \epsilon_2 = 0.5)$ and $(\epsilon_\infty = \frac{8}{255}, \mathbf{\epsilon_1 = 18}, \epsilon_2 = 0.5)$ with robust fine-tuning.}
    \begin{minipage}{.5\linewidth}
\centering
\vspace{2mm}
\begin{tabular}{lrrrrr}
     & Clean & $l_\infty$ & $l_2$   & $l_1$   & Union         \\\hline
E-AT & 81.3  & 43.5 & 66.6 & 42.8 & 39.0          \\
% MAX  & 80.6  & 44.0 & 66.3 & 44.1 & 40.8          \\
\bf RAMP & 80.4  & 44.2 & 66.1 & 44.4 & \textbf{41.2} \\\hline
\end{tabular}
\end{minipage}%
    \begin{minipage}{.5\linewidth}
\centering
\vspace{2mm}
\begin{tabular}{lrrrrr}
     & Clean & $l_\infty$ & $l_2$   & $l_1$   & Union         \\\hline
E-AT & 81.3  & 38.9 & 66.6 & 45.0 & 37.5          \\
% MAX  & 80.9  & 40.4 & 66.4 & 43.3 & 38.7          \\
\bf RAMP & 80.7  & 40.6 & 66.3 & 43.5 & \textbf{38.8}\\\hline
\end{tabular}
\end{minipage}
\label{table:finetune-l12}
\end{table}

\noindent\textbf{Changing $l_2$ perturbations with $\epsilon_2 = [0.25, 0.75, 1.0, 1.5]$.} Table~\ref{table:scratch-l21} and Table~\ref{table:scratch-l22} show that \textbf{RAMP} consistently outperforms E-AT~\citep{croce2022eat} on union accuracy when performing robust fine-tuning.

\begin{table}[h!]
\caption{$(\epsilon_\infty = \frac{8}{255}, \epsilon_1 = 12, \mathbf{\epsilon_2 = 0.25})$ and $(\epsilon_\infty = \frac{8}{255}, \epsilon_1 = 12, \mathbf{\epsilon_2 = 0.75})$ with robust fine-tuning.}
    \begin{minipage}{.5\linewidth}
\centering
\vspace{2mm}
\begin{tabular}{lrrrrr}
     & Clean & $l_\infty$ & $l_2$   & $l_1$   & Union         \\\hline
E-AT & 82.3  & 44.2 & 75.3 & 47.2 & 41.4          \\
% MAX  & 81.7  & 45.1 & 74.8 & 46.6 & 42.3          \\
\bf RAMP & 81.5  & 45.6 & 74.4 & 47.1 & \textbf{43.1}\\\hline
\end{tabular}
\end{minipage}%
    \begin{minipage}{.5\linewidth}
\centering
\vspace{2mm}
\begin{tabular}{lrrrrr}
     & Clean & $l_\infty$ & $l_2$   & $l_1$   & Union         \\\hline
E-AT & 83.0  & 43.5 & 58.1 & 46.5 & 40.4          \\
% MAX  & 81.7  & 45.1 & 57.4 & 46.6 & 42.3          \\
\bf RAMP & 81.4  & 45.6 & 57.4 & 47.2 & \textbf{42.9}\\\hline
\end{tabular}
\end{minipage}
\label{table:finetune-l21}
\end{table}

\begin{table}[h!]
\caption{$(\epsilon_\infty = \frac{8}{255}, \epsilon_1 = 12, \mathbf{\epsilon_2 = 1.0})$ and $(\epsilon_\infty = \frac{8}{255}, \epsilon_1 = 12, \mathbf{\epsilon_2 = 1.5})$ with robust fine-tuning.}
    \begin{minipage}{.5\linewidth}
\centering
\vspace{2mm}
\begin{tabular}{lrrrrr}
     & Clean & $l_\infty$ & $l_2$   & $l_1$   & Union         \\\hline
E-AT & 82.3  & 41.0 & 49.0 & 51.6 & 40.2          \\
% MAX  & 81.8  & 42.8 & 48.8 & 48.5 & 41.4          \\
\bf RAMP & 81.4  & 45.6 & 47.8 & 47.1 & \textbf{42.9} \\\hline
\end{tabular}
\end{minipage}%
    \begin{minipage}{.5\linewidth}
\centering
\vspace{2mm}
\begin{tabular}{lrrrrr}
     & Clean & $l_\infty$ & $l_2$   & $l_1$   & Union         \\\hline
E-AT & 80.2  & 42.8 & 31.5 & 52.4 & 31.5          \\
% MAX  & 74.8  & 43.8 & 36.7 & 50.2 & 36.6          \\
\bf RAMP & 74.9  & 43.7 & 37.0 & 50.2 & \textbf{36.9}\\\hline
\end{tabular}
\end{minipage}
\label{table:finetune-l22}
\end{table}

\subsection{Different Logit Pairing Methods} In this section, we test \textbf{RAMP} with robust fine-tuning using two more different logit pairing losses: (1) Mean Squared Error Loss ($\mathcal{L}_{mse}$) (Eq.~\ref{eq:11}), (2) Cosine-Similarity Loss ($\mathcal{L}_{cos}$) (Eq.~\ref{eq:12}). We replace the KL loss we used in the paper using the following losses. We use the same lambda value $\lambda=1.5$ for both cases.

\begin{equation}\label{eq:11}
    \mathcal{L}_{mse} = \frac{1}{n_c} \cdot \sum^{n_c}_{i=0} \frac{1}{2} \left(p_{q}[\gamma[i]] - p_{r}[\gamma[i]]\right)^2
\end{equation}

\begin{equation}\label{eq:12}
    \mathcal{L}_{cos} = \frac{1}{n_c} \cdot \sum^{n_c}_{i=0} \left(1 - \cos (p_{q}[\gamma[i]], p_{r}[\gamma[i]])\right)
\end{equation}

Table~\ref{table:logit-pairing-losses} displays \textbf{RAMP} robust fine-tuning results of different logit pairing losses using PreAct-ResNet-18 on CIFAR-10 with $\lambda = 1.5$. We see those losses generally improve union accuracy compared with baselines in Table~\ref{table:resnet18-robust-finetune}. $\mathcal{L}_{cos}$ has a better clean accuracy yet slightly worsened union accuracy. $\mathcal{L}_{mse}$ has the best union accuracy and the worst clean accuracy. $\mathcal{L}_{KL}$ is in the middle of the two others. However, we acknowledge the possibility that each logit pairing loss may have its own best-tuned $\lambda$ value.

\begin{table}[h!]
\centering
\caption{\textbf{RAMP} fine-tuning results of different logit pairing losses using PreAct-ResNet-18 on CIFAR-10.}
% \small
% \fontsize{8}{10}\selectfont
\begin{tabular}{lrrrrr}
   Losses                & {Clean} & $l_\infty$ & $l_2$ & $l_1$ & Union \\\hline
KL & 80.9             & \textbf{45.5}            & 66.2        & 47.3          & 43.1            \\
MSE & 80.4                      & \bf45.6                     & 65.8                   & \bf47.6                   & \bf43.5\\
Cosine &         \bf81.6              &     45.4            &    \bf66.7         &    47.0         & 42.9\\\hline                     
\end{tabular}

\label{table:logit-pairing-losses}
\end{table}

\subsection{AT from Scratch Using WideResNet-28-10}\label{wideresnet-exp}
\noindent\textbf{Implementations.} We use a cyclic learning rate with a maximum rate of $0.1$ for $30$ epochs and adopt the outer minimization trades loss from~\citet{zhang2019trades} with the default hyperparameters, same as~\citet{croce2022eat}; also, we set $\lambda=2.0$ and $\beta = 0.5$ for training \textbf{RAMP}. Additionally, we use the WideResNet-28-10 architecture same as~\citet{zagoruyko2016wide} for our reimplementations on CIFAR-10. 

\noindent\textbf{Results.} Since the implementation of experiments on WideResNet-28-10 in~\citet{croce2022eat} paper is not public at present, we report our implementation results on E-AT, where our results show that \textbf{RAMP} outperforms E-AT in union accuracy with a significant margin, as shown in Table~\ref{table:wideresnet-robust-pretrain}. Also, we experiment with using the trade loss (\textbf{RAMP w trades}) for the outer minimization, we observe that \textbf{RAMP w trades} achieves a better union accuracy at the loss of some clean accuracy.

\begin{table}[!h]
    \centering
    % \small
       \caption{\textbf{WideResNet-28-10 trained from random initialization} on CIFAR-10. \textbf{RAMP} outperforms E-AT on union accuracy with our implementation.}
    % \fontsize{8}{10}\selectfont
    \begin{tabular}{lrrrrr}
    \text {Methods} & \text { Clean } & $l_{\infty}$ & $l_2$ & $l_1$ &  \text { Union }\\
\hline 
% SAT & $80.5 \pm 0.6$ & $45.9 \pm 0.5$ & $66.7 \pm 0.3$ & $55.9 \pm 0.5$ & $56.2 \pm 0.4$ & $45.7 \pm 0.6$ \\
% MNG-AC & $81.3 \pm 0.3$ & $43.5 \pm 0.7$ & $66.9 \pm 0.2$ & $57.6 \pm 0.8$ & $56.0 \pm 0.4$ & $43.3 \pm 0.7$ \\
% AVG & $82.5 \pm 0.4$ & $45.4 \pm 1.1$ & $68.0 \pm 0.9$ & $55.0 \pm 0.2$ & $56.1 \pm 0.7$ & $45.1 \pm 1.1$ \\
% MAX & $79.9 \pm 0.1$ & $48.4 \pm 0.7$ & $65.3 \pm 0.3$ & $50.2 \pm 0.6$ & $54.6 \pm 0.5$ & $47.4 \pm 0.8$ \\
% MSD & $80.6 \pm 0.3$ & $48.0 \pm 0.2$ & $65.6 \pm 0.3$ & $51.7 \pm 0.4$ & $55.1 \pm 0.2$ & $46.9 \pm 0.1$ \\
% E-AT unif. & $79.7 \pm 0.2$ & $45.4 \pm 0.5$ & $66.0 \pm 0.5$ & $55.6 \pm 0.5$ & $55.7 \pm 0.4$ & $45.1 \pm 0.7$ \\
E-AT w trades (reported in~\citet{croce2022eat}) & 79.9 & 46.6 & 66.2 & 56.0 & 46.4 \\
E-AT w trades (ours) & 79.2 & 44.2 & 64.9 & 54.9 & 44.0\\
% \textbf{RAMP w/o GP} (ours) & 77.9 & 46.1 & 63.0 & 49.3 & \bf45.2\\
% MAX w/o trades (ours) & 83.3 & 45.1 & 66.8 & 48.1 & 43.5\\
\textbf{RAMP w/o trades} (ours) & 81.1 & 46.6 & 65.9 & 48.1 & 44.6 \\
\textbf{RAMP w trades} (ours) & 79.9 & 47.1 & 65.1 & 49.0 & \bf45.8\\
\hline
\end{tabular}
 
    \label{table:wideresnet-robust-pretrain}
\end{table}

\subsection{Robust Fine-tuning Using PreAct-ResNet-18}\label{robust-finetune-rn18-exp}
\textbf{Implementations.} For robust fine-tuning with ResNet-18, we perform $3$ epochs on CIFAR-10. We set the learning rate as $0.05$ for PreAct-ResNet-18 and $0.01$ for other models. We set $\lambda=0.5$ in this case. Also, we reduce the learning rate by a factor of $10$ after completing each epoch.

\noindent\textbf{Result.} Table~\ref{table:resnet18-robust-finetune} shows the robust fine-tuning results using PreAct ResNet-18 model on the CIFAR-10 dataset with different methods. The results for all baselines are directly from the E-AT paper~\citep{croce2022eat} where the authors reimplemented other baselines (e.g., MSD, MAX) to achieve better union accuracy than presented in the original works. \textbf{RAMP} surpasses all other methods on union accuracy. 

\begin{table*}[!h]
    \centering
       \caption{\textbf{RN-18 $l_\infty$-AT model fine-tuned} for 3 epochs (repeated for 5 seeds). \textbf{RAMP} has the highest union accuracy. Baseline results are from~\citet{croce2022eat}.}
    % \small
    % \fontsize{8}{10}\selectfont
    % \fontsize{7.8}{10}\selectfont
    \begin{tabular}{lrrrrr}
\text {Methods} & \text { Clean } & $l_{\infty}$ & $l_2$ & $l_1$ & \text { Union }\\
\hline RN-18- $l_{\infty}$-AT & $83.7$ & $48.1$ & $59.8$ & $7.7$ & $38.5$ \\
+ SAT & $83.5 \pm 0.2$ & $43.5 \pm 0.2$ & $68.0 \pm 0.4$ & $47.4 \pm 0.5$ & $41.0 \pm 0.3$ \\
+ AVG & $84.2 \pm 0.4$ & $43.3 \pm 0.4$ & $68.4 \pm 0.6$ & $46.9 \pm 0.6$ & $40.6 \pm 0.4$\\
+ MAX & $82.2 \pm 0.3$ & $45.2 \pm 0.4$ & $67.0 \pm 0.7$ & $46.1 \pm 0.4$ & $42.2 \pm 0.6$\\
+ MSD & $82.2 \pm 0.4$ & $44.9 \pm 0.3$ & $67.1 \pm 0.6$ & $47.2 \pm 0.6$ & $42.6 \pm 0.2$ \\
+ E-AT & $82.7 \pm 0.4$ & $44.3 \pm 0.6$ & $68.1 \pm 0.5$ & $48.7 \pm 0.5$ & $42.2 \pm 0.8$ \\
+ \textbf{RAMP} ($\lambda$=1.5) & $81.1 \pm 0.2$ & $45.4 \pm 0.3$ & $66.1 \pm 0.2$ & $47.2 \pm 0.1$ & $\mathbf{43.1 \pm 0.2}$ \\
+ \textbf{RAMP} ($\lambda$ = 0.5) & $81.5 \pm 0.1$ & $45.5 \pm 0.2$ & $66.4 \pm 0.2$ & $47.0 \pm 0.1$ & $\mathbf{42.9 \pm 0.2}$ \\\hline
\end{tabular}
 
    % Time per epoch is evaluated using a single Tesla V100 GPU.
    \label{table:resnet18-robust-finetune}
\end{table*}

\subsection{Robust Fine-tuning with More Epochs}\label{robust-finetune-epochs}
In Table~\ref{table:finetune-more-epochs}, we apply robust fine-tuning on the PreAct ResNet-18 model for the CIFAR-10 dataset with $5,7,10,15$ epochs, and compare it with E-AT. \textbf{RAMP} consistently outperforms the baseline on union accuracy, with a larger improvement when we increase the number of epochs.
\begin{table}
\centering
\caption{\textbf{Fine-tuning with more epochs}: \textbf{RAMP} consistently outperforms E-AT on union accuracy. E-AT results are from~\citet{croce2022eat}.}
% \fontsize{5.5}{5.5}\selectfont
% \fontsize{8}{10}\selectfont
\begin{tabular}{@{}lllllllll@{}}
              & \multicolumn{2}{c}{5 epochs}                                 & \multicolumn{2}{c}{7 epochs}                                 & \multicolumn{2}{c}{10 epochs}                                & \multicolumn{2}{c}{15 epochs}                                \\\hline
              & \multicolumn{1}{r}{Clean} & \multicolumn{1}{l}{Union} & \multicolumn{1}{r}{Clean} & \multicolumn{1}{l}{Union} & \multicolumn{1}{r}{Clean} & \multicolumn{1}{l}{Union} & \multicolumn{1}{r}{Clean} & \multicolumn{1}{l}{Union} \\\hline
E-AT          & 83.0                      & 43.1                      & 83.1                      & 42.6                      & 84.0                      & 42.8                      & 84.6                      & 43.2                      \\
\textbf{RAMP} & 81.7                      & \textbf{43.6}             & 82.1                      & \textbf{43.8}             & 82.5                      & \textbf{44.6}             & 83.0                      & \textbf{44.9}  \\\hline     
\end{tabular}

\label{table:finetune-more-epochs}
\end{table}

% \subsection{Discussions on the Choices of Epsilons} For given epsilon values of $\epsilon_1$, $\epsilon_2$, $\epsilon_\infty$ and $d$ (dimension of the inputs), we need the following relationships to make three perturbations intersect with each other:

% $$0 \leq \epsilon_1 \leq d \cdot \epsilon_\infty$$
% $$0 \leq \epsilon_2 \leq \sqrt{d} \cdot \epsilon_\infty$$
% $$\frac{\epsilon_1}{\sqrt{d}} \leq \epsilon_2 \leq \epsilon_1$$

% Then we check for our choices of CIFAR10 $\epsilon_1=12, \epsilon_2=0.5, \epsilon_\infty=\frac{8}{255}, d = 3 \cdot 32^2$ and Imagenet $\epsilon_1=255, \epsilon_2=2, \epsilon_\infty=\frac{4}{255}, d = 3 \cdot 224^2$. For CIFAR10, we have $0 \leq 12 \leq 3072\cdot\frac{8}{255}, 0 \leq 0.5 \leq \sqrt{3072}\cdot\frac{8}{255}, \frac{12}{\sqrt{3072}} \leq 0.5 \leq 12$ - all of three holds, thus our choice makes three perturbations intersect with each other. For Imagenet, we have $0 \leq 255 \leq 150528\cdot\frac{4}{255}, 0 \leq 2 \leq \sqrt{150528}\cdot\frac{4}{255}, \frac{255}{\sqrt{150528}} \leq 2 \leq 255$ - all of three holds, thus our choice makes three perturbations intersect with each other.

\section{Additional Visualization Results}
In this section, we provide additional t-SNE visualizations of the multiple-norm tradeoff and robust fine-tuning procedures using different methods.

\subsection{Pre-trained $l_1, l_2, l_\infty$ AT models}
Figure~\ref{fig:full-pretrain} shows the robust accuracy of $l_1, l_2, l_\infty$ AT models against their respect $l_1, l_2, l_\infty$ perturbations, on CIFAR-10 using PreAct-ResNet-18 architecture. Similar to Figure~\ref{fig:l1-linf-visual}, $l_\infty$-AT model has a low $l_1$ robustness and vice versa. In this common choice of epsilons, we further confirm that $l_\infty - l_1$ is the key trade-off pair.

\subsection{Robust Fine-tuning for all Epochs}
We provide the complete visualizations of robust fine-tuning for 3 epochs on CIFAR-10 using $l_1$ examples, E-AT, and \textbf{RAMP}. Rows in $l_1$ fine-tuning (Figure~\ref{fig:full-finetune-l1}), E-AT fine-tuning (Figure~\ref{fig:full-finetune-eat}), and \textbf{RAMP} fine-tuning (Figure~\ref{fig:full-finetune-ramp}) show the robust accuracy against $l_\infty, l_1, l_2$ attacks individually, of epoch $0,1,2,3$, respectively. We observe that throughout the procedure, \textbf{RAMP} manages to maintain more $l_\infty$ robustness during the fine-tuning with more points colored in cyan, in comparison with two other methods. This visualization confirms that after we identify a $l_p - l_r (p\neq r)$ key tradeoff pair, \textbf{RAMP} successfully preserves more $l_p$ robustness when training with some $l_r$ examples via enforcing union predictions with the logit pairing loss.

\begin{figure}[!h]
    \centering
    \includegraphics[width=0.5\columnwidth]{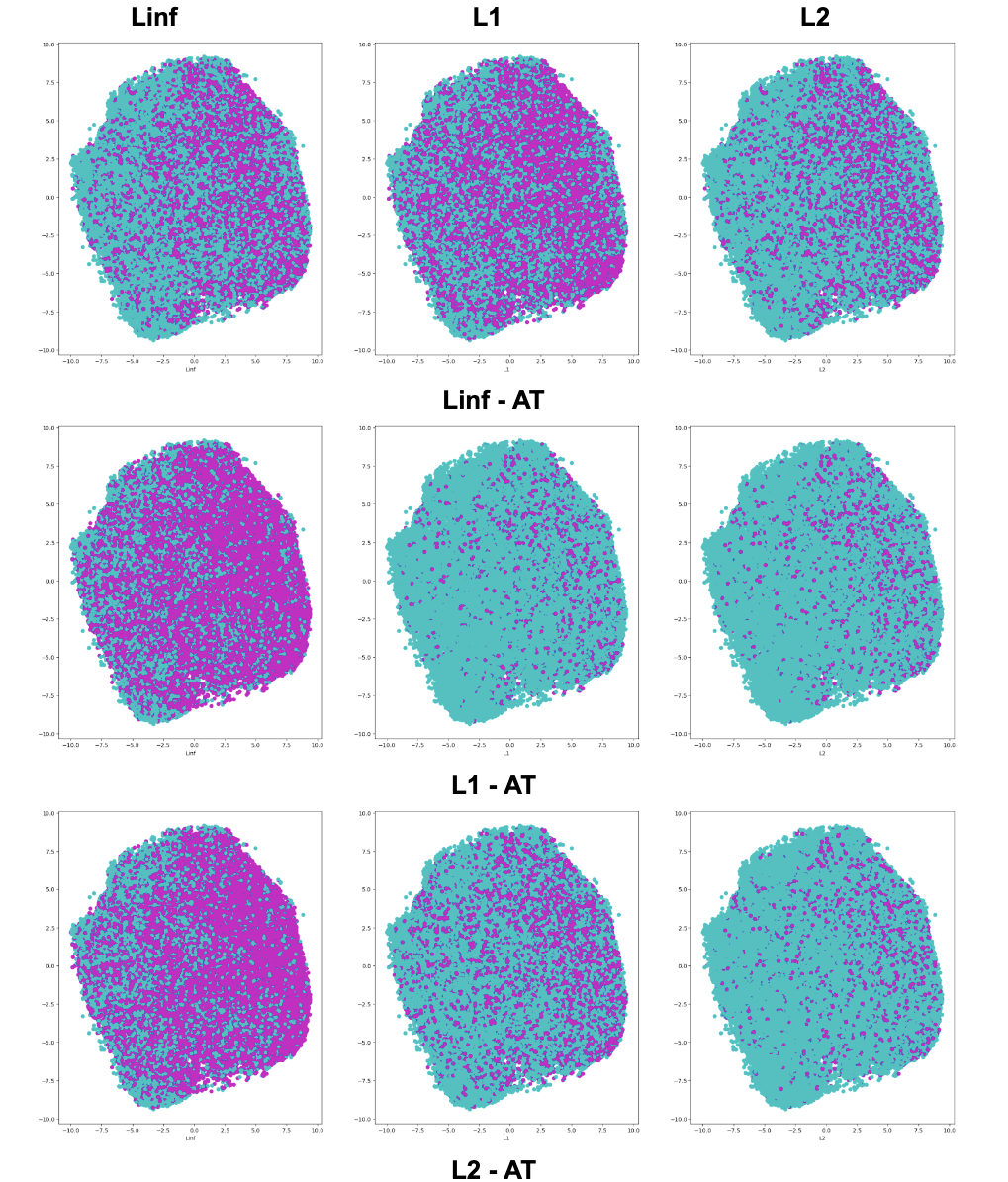}
    \caption{$l_1, l_2, l_\infty$ pre-trained RN18 $l_\infty$-AT models with correct/incorrect predictions against $l_1, l_2, l_\infty$ attacks. Correct predictions are colored with cyan and incorrect with magenta. Each row represents $l_\infty, l_1, l_2$ AT models, respectively. Each column shows the accuracy concerning a certain $l_p$ attack.}
    \label{fig:full-pretrain}
\end{figure}

\begin{figure}[!h]
    \centering
    \includegraphics[width=0.5\columnwidth]{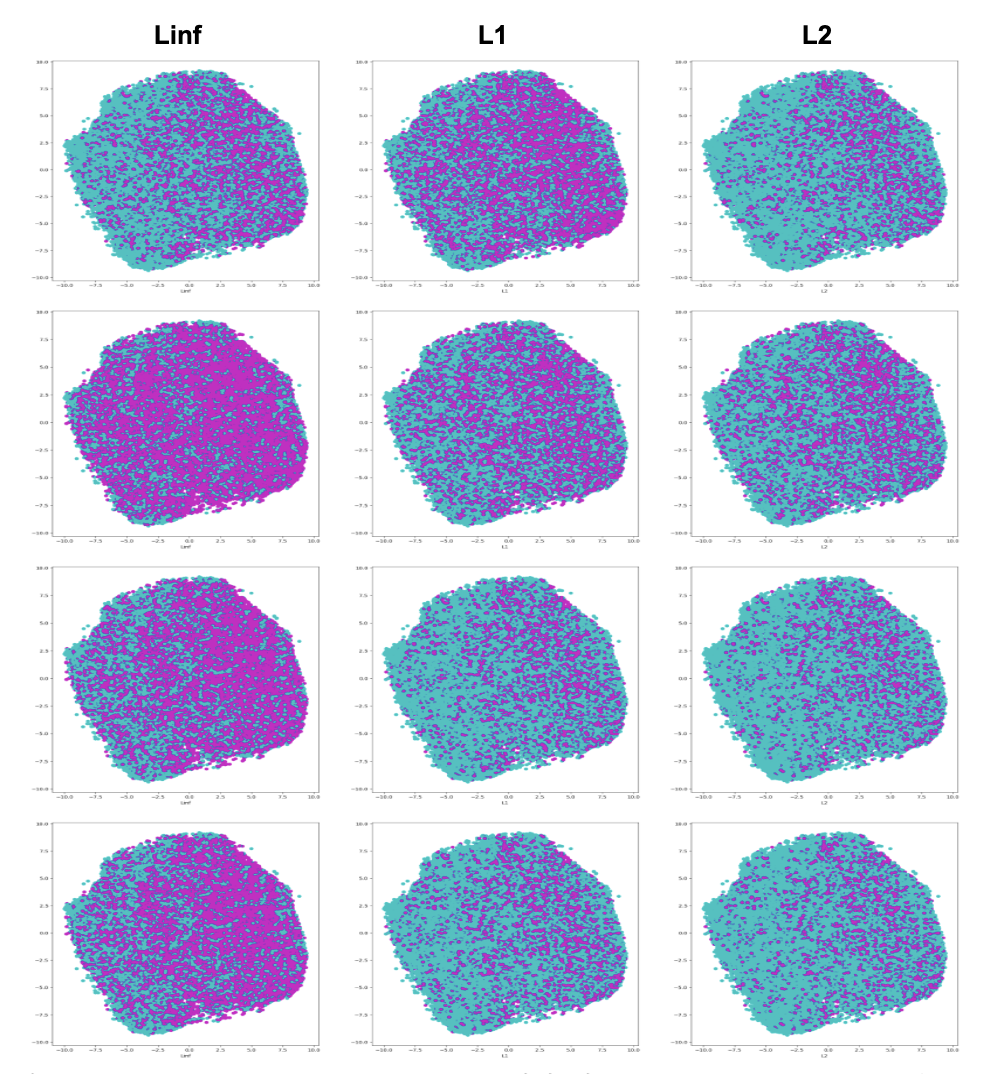}
    \caption{\textbf{Finetune RN18 $l_\infty$-AT model on $l_1$ examples for 3 epochs}. Each row represents the prediction results of epoch $0,1,2,3$ respectively.}
    \label{fig:full-finetune-l1}
\end{figure}

\begin{figure}[!h]
    \centering
    \includegraphics[width=0.5\columnwidth]{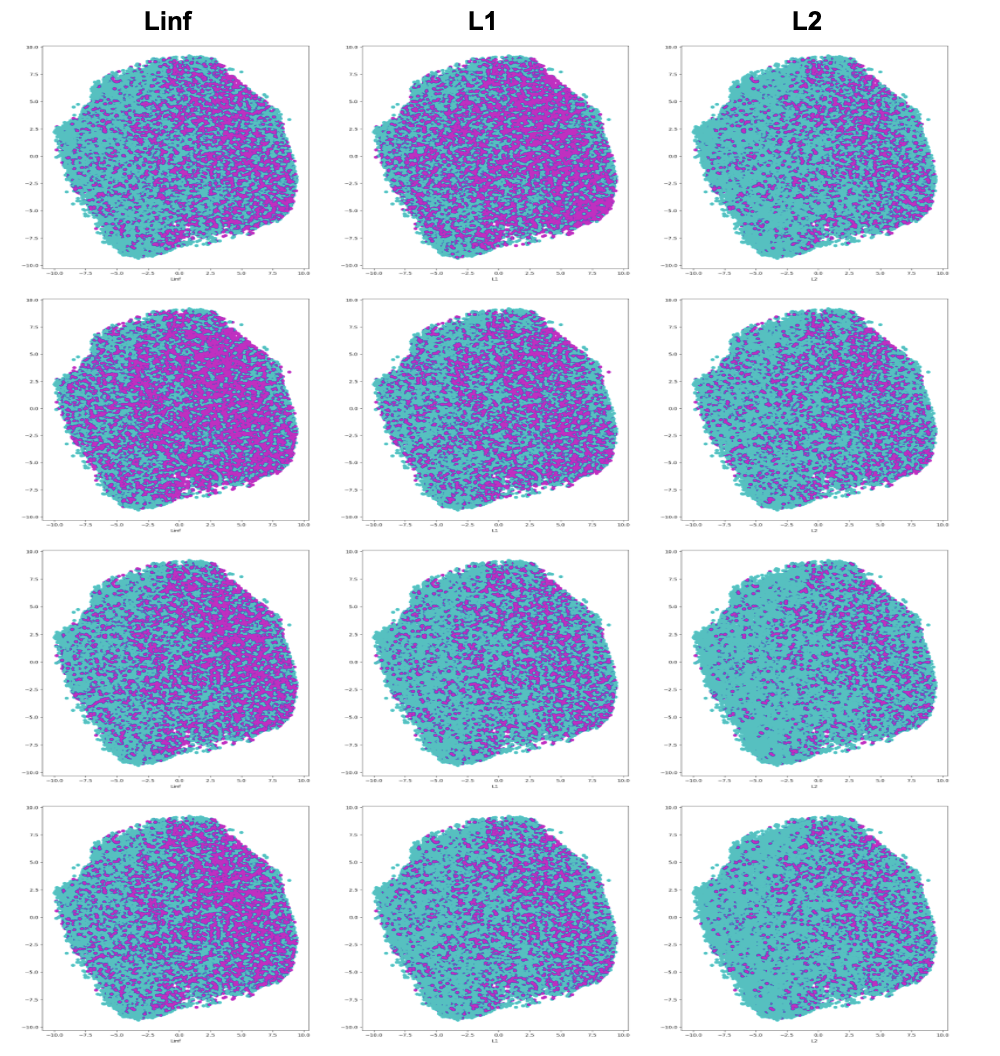}
    \caption{\textbf{Finetune RN18 $l_\infty$-AT model with E-AT for 3 epochs}. Each row represents the prediction results of epoch $0,1,2,3$ respectively.}
    \label{fig:full-finetune-eat}
\end{figure}

\begin{figure}
    \centering
    \includegraphics[width=0.5\columnwidth]{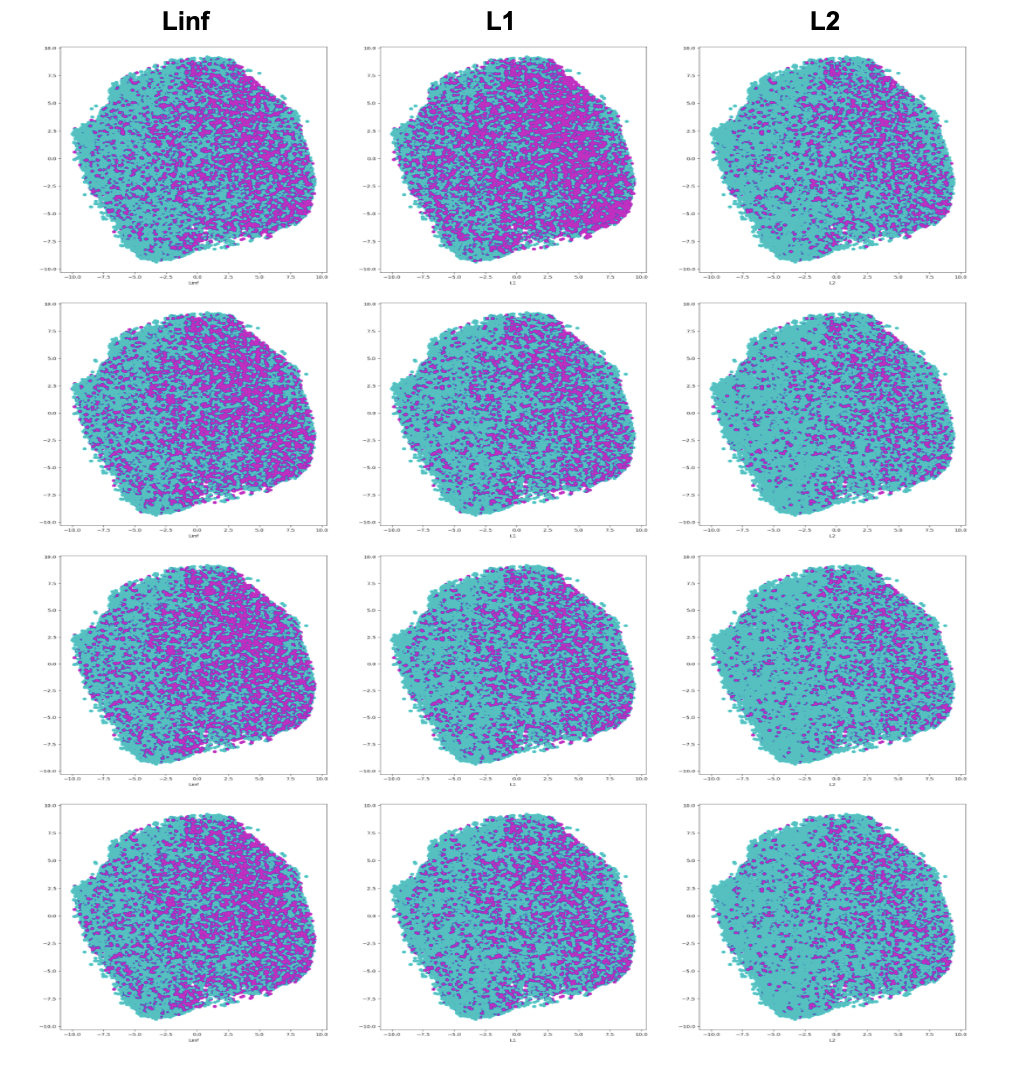}
    \caption{\textbf{Finetune RN18 $l_\infty$-AT model with \textbf{RAMP} for 3 epochs}. Each row represents the prediction results of epoch $0,1,2,3$ respectively.}
    \label{fig:full-finetune-ramp}
\end{figure}

\clearpage
\newpage

\end{document}